\documentclass{article}

    \PassOptionsToPackage{square,comma,numbers,sort}{natbib}
\usepackage[final]{neurips_2020}

\usepackage[utf8]{inputenc} % allow utf-8 input
\usepackage[T1]{fontenc}    % use 8-bit T1 fonts
\usepackage{url}            % simple URL typesetting
\usepackage{booktabs}       % professional-quality tables
\usepackage{amsfonts}       % blackboard math symbols
\usepackage{nicefrac}       % compact symbols for 1/2, etc.
\usepackage{microtype}      % microtypography

\usepackage{xspace}
\usepackage{amsmath,amsfonts,amssymb,amsthm}
\usepackage{bm}
\usepackage[dvipsnames]{xcolor}
\usepackage{diagbox}
\usepackage{pifont}
\usepackage{svg}
\usepackage{amsmath}
\usepackage{natbib}
\usepackage{enumitem}
\usepackage{mathtools}
\usepackage[colorlinks,linkcolor=blue]{hyperref}
\usepackage{algorithm}
\usepackage{algorithmic}
\usepackage{gensymb}
\usepackage{wrapfig}
\usepackage{subcaption}
\usepackage{multirow}

\newtheorem{theorem}{Theorem}
\newtheorem*{theorem*}{Theorem}

\newtheorem*{remark}{Remark}
\newtheorem{assumption}{Assumption}

\newcommand{\cmark}{\ding{51}}%
\newcommand{\xmark}{\ding{55}}%

\DeclarePairedDelimiterX{\infdivx}[2]{(}{)}{%
  #1\;\delimsize\|\;#2%
}
\DeclarePairedDelimiter{\norm}{\lVert}{\rVert}

\DeclareMathOperator*{\argmax}{arg\,max}
\DeclareMathOperator*{\argmin}{arg\,min}

\makeatletter
\DeclareRobustCommand\onedot{\futurelet\@let@token\@onedot}
\def\@onedot{\ifx\@let@token.\else.\null\fi\xspace}

\def\eg{\emph{e.g}\onedot} 
\def\ie{\emph{i.e}\onedot} 
 
\def\etc{\emph{etc}\onedot} 
\def\wrt{w.r.t\onedot}

\makeatother

\def\a{\bm{\alpha}}
\def\ahat{\hat{\bm{\alpha}}}
\def\astar{\bm{\alpha}^\ast}
\def\adagger{\bm{\alpha}^\dagger}
\def\Sbar{\bar{S}}
\def\KLa{KL_{\a}}
\def\KLastar{KL_{\astar}}
\def\KLahat{KL_{\ahat}}
\def\KLadagger{KL_{\adagger}}
\def\Sint{\int_{\theta \in S}}
\def\Sbarint{\int_{\theta \in \Sbar}}

\title{Auxiliary Task Reweighting for \\ Minimum-data Learning}
\author{%
  Baifeng Shi \\
  Peking University\\
  \texttt{bfshi@pku.edu.cn} \\
  \And
  Judy Hoffman \\
  Georgia Institute of Technology \\
  \texttt{judy@gatech.edu} \\
  \AND
  Kate Saenko \\
  Boston University \& MIT-IBM Watson AI Lab\\
  \texttt{saenko@bu.edu} \\
  \AND
  Trevor Darrell, Huijuan Xu \\
  University of California, Berkeley \\
  \texttt{\{trevor, huijuan\}@eecs.berkeley.edu} \\
}

\begin{document}

\maketitle

\begin{abstract}
  %Supervised learning is known to be data-hungry, preventing its application in situations where labeled data is limited or hard to collect. To compensate for data scarcity, auxiliary tasks are often used to provide additional supervision. However, to get the most benefits, one normally has to tune the importance weights for different tasks, which is extremely time-consuming. In this work, we propose a method to automatically reweight auxiliary tasks so that data requirement on main task is minimized. Specifically, we use weighted posterior of auxiliary tasks as a surrogate prior for the main task. Then the amount of information required for training is minimized by adjusting the weights and minimizing the divergence between surrogate prior and ground truth prior. In multiple experimental settings, we demonstrate our algorithm can more effectively utilize labeled data compared with previous methods. Surprisingly, we find that even in unsupervised settings, when provided with little extra labeled data (\eg 1 images per class), our algorithm can bring a non-trivial improvement.
  
  %Supervised learning is known to be data-hungry, preventing its application in situations where labeled data is limited or hard to collect. 
  Supervised learning requires a large amount of training data, limiting its application where labeled data is scarce.
  To compensate for data scarcity, one possible method is to utilize auxiliary tasks to provide additional supervision for the main task.
  Assigning and optimizing the importance weights for different auxiliary tasks remains an crucial and largely understudied research question.
  In this work, we propose a method to automatically reweight auxiliary tasks in order to reduce the data requirement on the main task.
  Specifically, we formulate the weighted likelihood function of auxiliary tasks as a surrogate prior for the main task.
  By adjusting the auxiliary task weights to minimize the divergence between the surrogate prior and the true prior of the main task, we obtain a more accurate prior estimation, achieving the goal of minimizing the required amount of training data for the main task and avoiding a costly grid search. 
  In multiple experimental settings (\eg semi-supervised learning, multi-label classification), we demonstrate that our algorithm can effectively utilize limited labeled data of the main task with the benefit of auxiliary tasks compared with previous task reweighting methods. 
  We also show that under extreme cases with only a few extra examples (\eg few-shot domain adaptation), our algorithm results in significant improvement over the baseline. Our code and video is available at \url{https://sites.google.com/view/auxiliary-task-reweighting}.
\end{abstract}

%%%%%%%%%%%%%%%%%%%%%%%%%%%%%%%%%%%%%%%%%%%%%%%%%%%%%%%%%%%%

\section{Introduction}

% Supervised learning has enjoyed its success, especially with recent development of deep neural networks, in areas of vision~\cite{he2016deep}, language~\cite{devlin2018bert}, \etc. However, purely supervised methods often need an enormous amount of labeled data, which is generally hard to collect or even unavailable in some cases. To this end, various settings and algorithms are designed for more efficient learning, including semi-supervised learning~\cite{tarvainen2017mean,oliver2018realistic}, transfer learning~\cite{tzeng2015simultaneous}, few-shot learning~\cite{chen2019closer}, \etc.

Supervised deep learning methods typically require an enormous amount of labeled data, which for many applications, is difficult, time-consuming, expensive, or even impossible to collect. As a result, there is a significant amount of research effort devoted to efficient learning with limited labeled data, including semi-supervised learning~\cite{tarvainen2017mean,oliver2018realistic}, transfer learning~\cite{tzeng2015simultaneous}, few-shot learning~\cite{chen2019closer}, domain adaptation~\cite{tzeng2017adversarial}, and representation learning~\cite{oord2018representation}.

% Among different approaches, auxiliary tasks are widely used to address the lack of data by providing additional supervision. The common practice is to train the main task and auxiliary tasks jointly. Note that despite joint training, we only care about the performance on the main task. Auxiliary tasks are either collected from related tasks or domains where we have abundant data~\cite{tzeng2017adversarial}, or manually designed to fit the latent data structure~\cite{zhai2019s4l,shelhamer2016loss}. Training with auxiliary tasks is demonstrated to achieve better generalization~\cite{ando2005framework}, hence is widely used in semi-supervised learning~\cite{zhai2019s4l}, self-supervised learning~\cite{oord2018representation}, transfer learning~\cite{tzeng2015simultaneous}, reinforcement learning~\cite{jaderberg2016reinforcement}, \etc.

Among these different approaches, auxiliary tasks are widely used to alleviate the lack of data by providing additional supervision, \ie using the same data or auxiliary data for a different learning task during the training procedure. Auxiliary tasks are usually collected from related tasks or domains where there is abundant data~\cite{tzeng2017adversarial}, or manually designed to fit the latent data structure~\cite{zhai2019s4l,shelhamer2016loss}. Training with auxiliary tasks has been shown to achieve better generalization~\cite{ando2005framework}, and is therefore widely used in many applications, \eg semi-supervised learning~\cite{zhai2019s4l}, self-supervised learning~\cite{oord2018representation}, transfer learning~\cite{tzeng2015simultaneous}, and reinforcement learning~\cite{jaderberg2016reinforcement}.

% Though seemingly attractive, one big challenge in using auxiliary tasks is how to choose the ``right'' tasks. Here the ``right'' tasks refer to the ones which can provide supervision highly related to the main task, so that we will need less data from main task to train a decent model. Among all kinds of choices, some tasks are more related to the main task, and some are less. Some may even impair the performance on the main task. One simple screening strategy is to compare the main task performance when training with each auxiliary task separately~\cite{zamir2018taskonomy}. However, in this way we have to enumerate all the candidate tasks, which is expensive especially when the candidate pool is large. Besides, it also ignores the model behaviour when combining multiple tasks together. A more sophisticated way would be training with all auxiliary tasks together in one pass, then using some algorithm to automatically determine the importance weights of each task. There have been some attempts towards this goal~\cite{chen2017gradnorm,lin2019adaptive,bakker2003task,du2018adapting}, but they either make strong assumptions that all related tasks are equally important, or have different motivations (\eg speeding up training) (See Sec.~\ref{sec:related_work} for more details).
Usually both the main task and auxiliary task are jointly trained, but only the main task's performance is important for the downstream goals. The auxiliary tasks should be able to reduce the amount of labeled data required to achieve a given performance for the main task. However, this has proven to be a difficult selection problem as certain seemingly related auxiliary tasks yield little or no improvement for the main task. One simple task selection strategy is to compare the main task performance when training with each auxiliary task separately~\cite{zamir2018taskonomy}. However, this requires an exhaustive enumeration of all candidate tasks, which is prohibitively expensive when the candidate pool is large. Furthermore, individual tasks may behave unexpectedly when combined together for final training. Another strategy is training all auxiliary tasks together in a single pass and using an evaluation technique or algorithm to automatically determine the importance weight for each task. There are several works along this direction~\cite{chen2017gradnorm,lin2019adaptive,bakker2003task,du2018adapting}, but they either only filter out unrelated tasks without further differentiating among related ones, or have a focused motivation (\eg faster training) limiting their general use.

% In this work, we propose a method to adaptively reweight auxiliary tasks on the fly during joint training. Specifically, our algorithm finds the optimal weights for each auxiliary task so that the data requirement on the main task is minimized. We start from a key insight: \emph{we can reduce the data requirement by choosing a high-quality prior}. Then by noticing that joint training is basically using the parameter distribution induced by data likelihood of auxiliary tasks as a surrogate prior for the main task, we adjust the weights so that the ``distance'' between surrogate prior and true prior is minimized. We show in various experimental settings that our method can make better use of labeled data and effectively reduce the data requirement for the main task. Surprisingly, we find that very little labeled data (\eg 1 images per class) is enough for our algorithm to bring a substantial improvement over the baseline.

%In this work, we propose a method to adaptively reweight auxiliary tasks on the fly during joint training. Specifically, our algorithm finds the optimal weights for each auxiliary task so that the data requirement on the main task is minimized. 
%Then, we show that the joint training effectively uses the parameter distribution induced by the auxiliary tasks' likelihood as a surrogate prior for the main task, which allows us to adjust the task weights to minimize the distance between the surrogate priors and true prior. 
In this work, we propose a method to adaptively reweight auxiliary tasks on the fly during joint training so that the data requirement on the main task is minimized. We start from a key insight: \emph{we can reduce the data requirement by choosing a high-quality prior}. 
Then we formulate the parameter distribution induced by the auxiliary tasks' likelihood as a surrogate prior for the main task.
By adjusting the auxiliary task weights, the divergence between the surrogate prior and the true prior of the main task is minimized. In this way, the data requirement on the main task is reduced under high quality surrogate prior. Specifically, due to the fact that minimizing the divergence is intractable, we turn the optimization problem into minimizing the distance between gradients of the main loss and the auxiliary losses, which allows us to design a practical, light-weight algorithm.
We show in various experimental settings that our method can make better use of labeled data and effectively reduce the data requirement for the main task. Surprisingly, we find that very little labeled data (\eg 1 image per class) is enough for our algorithm to bring a substantial improvement over unsupervised and few-shot baselines. %with no target labeled data in the few-shot domain adaptation setting.

%This paper is organized as follows: first we discuss related work in Sec.~\ref{sec:related_work}. In Sec.~\ref{sec:high_level_idea}, we present our high-level idea of reducing the data requirement through task reweighting. In Sec.~\ref{seq:approach}, we discuss obstacles to designing a practical algorithm and how to overcome them, and then give the final algorithm. Then we present the main experimental results in Sec.~\ref{sec:exp} and finally conclude in Sec.~\ref{sec:conclusion}.

\begin{figure}[t]
    \centering
    \includegraphics[width=0.85\columnwidth]{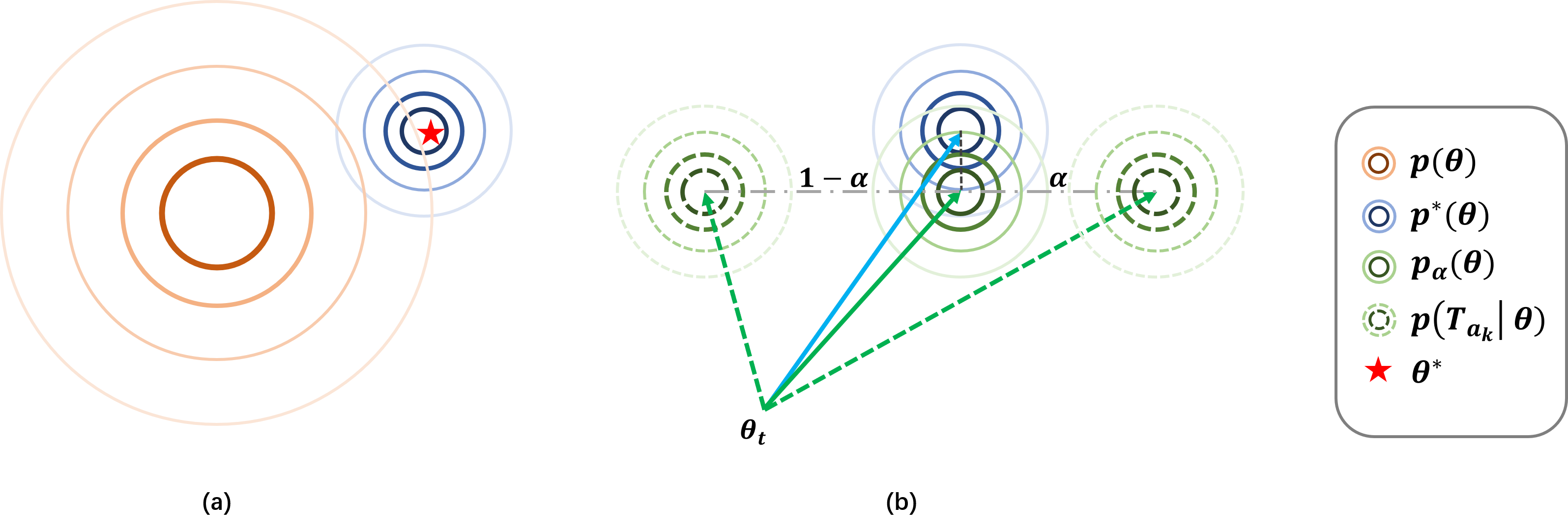}
    \caption{Learning with minimal data through auxiliary task reweighting. (a) An ordinary prior $p(\theta)$ over model parameters contains little information about the true prior $p^\ast(\theta)$ and the optimal parameter $\theta^\ast$. (b) Through a weighted combination of distributions induced by data likelihood $p(\mathcal{T}_{a_k} | \theta)$ of different auxiliary tasks, we find the optimal surrogate prior $p_\alpha (\theta)$ which is closest to the true prior. % For example, when distributions are Gaussian, this is equivalent to interpolating between means of different tasks and minimizing distance between score functions (see Sec.~\ref{seq:approach}) of the true prior (blue arrow) and surrogate prior (green solid arrow).
    }
    \label{fig:overview}
    \vspace{-0.5em}
\end{figure}
\section{Learning with Minimal Data through Auxiliary Task Reweighting}
\label{sec:ARML}

Suppose we have a main task with training data  $\mathcal{T}_m$ (including labels), and $K$ different auxiliary tasks with training data $\mathcal{T}_{a_k}$ for $k$-th task, where $k = 1, \cdots, K$. Our model contains a shared backbone with parameter $\theta$, and different heads for each task. Our goal is to find the optimal parameter $\theta^\ast$ for the main task, using data from main task as well as auxiliary tasks. Note that we care about performance on the main task and auxiliary tasks are only used to help train a better model on main task (\eg when we do not have enough data on the main task). In this section, we discuss how to learn with minimal data on main task by learning and reweighting auxiliary tasks.

\subsection{How Much Data Do We Need: Single-Task Scenario}
\label{sec:single_task_scenario}

Before discussing learning from multiple auxiliary tasks, we first start with the single-task scenario. When there is only one single task, we normally train a model by minimizing the following loss:
\begin{equation}
\label{eq:single_task_loss}
    \mathcal{L}(\theta) = -\log p(\mathcal{T}_m | \theta) -\log p(\theta) = -\log (p(\mathcal{T}_m | \theta) \cdot p(\theta)),
\end{equation}
where $p(\mathcal{T}_m | \theta)$ is the likelihood of training data and $p(\theta)$ is the prior. Usually, a relatively weak prior (\eg Gaussian prior when using weight decay) is chosen, reflecting our weak knowledge about the true parameter distribution, which we call `ordinary prior'. Meanwhile, we also assume there exists an unknown `true prior' $p^\ast(\theta)$ where the optimal parameter $\theta^\ast$ is actually sampled from. This true prior is normally more selective and informative (\eg having a small support set) (See Fig.~\hyperref[fig:overview]{1(a)}) ~\cite{baxter1997bayesian}.

%Then how much data do we need to learn the task? 
%In fact, the answer depends a lot on the choice of our prior $p(\theta)$.
Now our question is, how much data do we need to learn the task. Actually the answer depends on the choice of the prior $p(\theta)$. 
%Imagine if we know the `true prior' $p^\ast (\theta)$, then since $p^\ast(\theta)$ is really informative, only a few data items are needed to localize $\theta^\ast$ within the prior.
If we know the informative `true prior' $p^\ast (\theta)$, only a few data items are needed to localize the best parameters $\theta^\ast$ within the prior. 
However, if the prior is rather weak, we have to search $\theta$ in a larger space, which needs more data. Intuitively, the required amount of data is related to the divergence between $p(\theta)$ and $p^\ast(\theta)$: the closer they are, the less data we need.

In fact, it has been proven~\cite{baxter1997bayesian} that the expected amount of information needed to solve a single task is
\begin{equation}
    \mathcal{I} = D_{\mathrm{KL}}\infdivx{p^\ast}{p} + H(p^\ast), 
\end{equation}
where $D_{\mathrm{KL}}\infdivx{\cdot}{\cdot}$ is Kullback--Liebler divergence, and $H(\cdot)$ is the entropy. This means we can reduce the data requirement by choosing a prior closer to the true prior $p^\ast$. Suppose $p(\theta)$ is parameterized by $\alpha$, \ie, $p(\theta) = p_\alpha (\theta)$, then we can minimize data requirement by choosing $\alpha$ that satisfies:
\begin{equation}
\label{eq:objective_single_task}
    \min_{\alpha} D_{\mathrm{KL}}\infdivx{p^\ast}{p_\alpha}.
\end{equation}
However, due to our limited knowledge about the true prior $p^\ast$, it is unlikely to manually design a family of $p_\alpha$ that has a small value in~\eqref{eq:objective_single_task}. Instead, we will show that we can define $p_\alpha$ implicitly through auxiliary tasks, utilizing their natural connections to the main task.

\subsection{Auxiliary-Task Reweighting}
\label{sec:aux_task_reweightign}

When using auxiliary tasks, we optimize the following joint-training loss:
\begin{equation}
\label{eq:multi_task_loss}
    \mathcal{L}(\theta) = -\log p(\mathcal{T}_m | \theta) - \sum_{k=1}^K \alpha_k \log p(\mathcal{T}_{a_k} | \theta) = -\log (p(\mathcal{T}_M | \theta) \cdot \prod_{k=1}^K p^{\alpha_k}(\mathcal{T}_{a_k} | \theta)),
\end{equation}
where auxiliary losses are weighted by a set of task weights $\bm{\alpha} = (\alpha_1, \cdots, \alpha_K)$, and added together with the main loss. By comparing~\eqref{eq:multi_task_loss} with single-task loss~\eqref{eq:single_task_loss}, we can see that we are implicitly using $p_{\bm{\alpha}} (\theta) = \frac{1}{Z(\bm{\alpha})} \prod_{k=1}^K p^{\alpha_k}(\mathcal{T}_{a_k} | \theta)$ as a `surrogate' prior for the main task, where $Z(\bm{\alpha})$ is the normalization term (partition function). Therefore, as discussed in Sec.~\ref{sec:single_task_scenario}, if we adjust task weights $\bm{\alpha}$ towards
\begin{equation}
\label{eq:objective_multi_task}
    \min_{\bm{\alpha}} D_{\mathrm{KL}}\infdivx{p^\ast (\theta)}{\frac{1}{Z(\bm{\alpha})} \prod_{k=1}^K p^{\alpha_k}(\mathcal{T}_{a_k} | \theta)},
\end{equation}
then the data requirement on the main task can be minimized. This implies an automatic strategy of task reweighting. 
%Intuitively, if the parameter distribution from one task is closer to that of the main task, then this task can provide more information about the main task, thus having a higher weight.
Higher weights can be assigned to the auxiliary tasks with more relevant information to the main task, namely the parameter distribution of the tasks is closer to that of the main task. 
After taking the weighted combination of auxiliary tasks, the prior information is maximized, and the main task can be learned with minimal additional information (data). See Fig.~\hyperref[fig:overview]{1(b)} for an illustration.

\subsection{Our Approach}
\label{sec:approach}

In Sec.~\ref{sec:aux_task_reweightign} we have discussed about how to minimize the data requirement on the main task by reweighting and learning auxiliary tasks. However, the objective in~\eqref{eq:objective_multi_task} is hard to optimize directly due to a few practical problems:%, we still need some modifications before designing a practical algorithm. In this section, we will first list the obstacles to the optimization of~\eqref{eq:objective_multi_task}, then present how to overcome them using different tools or approximations, and finally propose our algorithm for task reweighting.

\begin{itemize}
    \item \textbf{True Prior (P1)}: We do not know the true prior $p^\ast$ in advance.
    \item \textbf{Samples (P2)}: KL divergence is in form of an expectation, which needs samples to estimate. However, sampling from a complex distribution is non-trivial.
    \item \textbf{Partition Function (P3)}: Partition function $Z(\bm{\alpha}) = \int \prod_{k=1}^K p^{\alpha_k}(\mathcal{T}_{a_k} | \theta) d\theta$ is given by an intractable integral,  preventing us from getting the accurate density function $p_{\bm{\alpha}}$.
\end{itemize}
To this end, we use different tools or approximations to design a practical algorithm, and keep its validity and effectiveness from both theoretical and empirical aspects, as presented below.

\paragraph{True Prior (P1)}
\label{para:true-prior}

In the original optimization problem~\eqref{eq:objective_multi_task}, we are minimizing
\begin{equation}
    D_{\mathrm{KL}}\infdivx{p^\ast (\theta)}{p_{\bm{\alpha}} (\theta)} = E_{\theta \sim p^\ast} \log \frac{p^\ast (\theta)}{p_{\bm{\alpha}} (\theta)},
\end{equation}
which is the expectation of $\log \frac{p^\ast (\theta)}{p_{\bm{\alpha}} (\theta)}$ \wrt{\ $p^\ast (\theta)$}. The problem is, $p^\ast(\theta)$ is not accessible. However, we can notice that for each $\theta^\ast$ sampled from prior $p^\ast$, it is likely to give a high data likelihood $p(\mathcal{T}_m | \theta^\ast)$, which means $p^\ast (\theta)$ is `covered' by $p(\mathcal{T}_m | \theta)$, \ie, $p(\mathcal{T}_m | \theta)$ has high density both in the support set of $p^\ast (\theta)$, and in some regions outside. Thus we propose to minimize $D_{\mathrm{KL}}\infdivx{p^m(\theta)}{p_{\bm{\alpha}} (\theta)}$ instead of $D_{\mathrm{KL}}\infdivx{p^\ast (\theta)}{p_{\bm{\alpha}} (\theta)}$, where $p^m(\theta)$ is the parameter distribution induced by data likelihood $p(\mathcal{T}_m | \theta)$, \ie, $p^m(\theta) \propto p(\mathcal{T}_m | \theta)$. Furthermore, we propose to take the expectation \wrt{\ $\frac{1}{Z^\prime(\bm{\alpha})} p^m(\theta) p_{\bm{\alpha}} (\theta)$} instead of $p^m(\theta)$ due to the convenience of sampling while optimizing the joint loss (see \hyperref[para:samples]{\textbf{P2}} for more details). Then our objective becomes
\begin{equation}
\label{eq:objective_multi_task_approx}
    \min_{\bm{\alpha}} E_{\theta \sim p^J} \log \frac{p^m(\theta)}{p_{\bm{\alpha}} (\theta)},
\end{equation}
where $p^J(\theta) = \frac{1}{Z^\prime(\bm{\alpha})} p^m(\theta) p_{\bm{\alpha}}(\theta)$, and $Z^\prime(\bm{\alpha})$ is the normalization term. 

Now we can minimize~\eqref{eq:objective_multi_task_approx} as a feasible surrogate for~\eqref{eq:objective_multi_task}. 
%although by doing so, we are also taking the risk that
However, minimizing~\eqref{eq:objective_multi_task_approx} may end up with a suboptimal $\bm{\alpha}$ for~\eqref{eq:objective_multi_task}. 
Due to the fact that $p^m(\theta)$ also covers some `overfitting area' other than $p^\ast (\theta)$, we may push $p_{\bm{\alpha}}(\theta)$ closer to the overfitting area instead of $p^\ast (\theta)$ by minimizing~\eqref{eq:objective_multi_task_approx}. 
But we prove that, under some mild conditions, if we choose $\bm{\alpha}$ that minimizes~\eqref{eq:objective_multi_task_approx}, the value of~\eqref{eq:objective_multi_task} is also bounded near the optimal value:
%However, under some mild conditions, we can prove that if we choose $\bm{\alpha}$ that minimizes~\eqref{eq:objective_multi_task_approx}, then the value of~\eqref{eq:objective_multi_task} is also bounded near the optimal value:

\begin{theorem}
\label{thrm:optimization_bound}
(Informal and simplified version) Let us denote the optimal weights for~\eqref{eq:objective_multi_task} and~\eqref{eq:objective_multi_task_approx} by $\bm{\alpha}^\ast$ and $\hat{\bm{\alpha}}$, respectively. Assume the true prior $p^\ast(\theta)$ has a small support set $S$. Let $\gamma = \max_{\bm{\alpha}} \int_{\theta \notin S} p^m(\theta) p_{\bm{\alpha}}(\theta) d\theta$ be the maximum of the integral of $p^m(\theta) p_{\bm{\alpha}}(\theta)$ outside $S$, then we have
\begin{equation}
    D_{\mathrm{KL}}\infdivx{p^\ast}{p_{\bm{\alpha}^\ast}} \leq D_{\mathrm{KL}}\infdivx{p^\ast}{p_{\hat{\bm{\alpha}}}} \leq  D_{\mathrm{KL}}\infdivx{p^\ast}{p_{\bm{\alpha}^\ast}} \ + \ C\gamma^2 \ - \ C^\prime \gamma^2 \log \gamma.
\end{equation}
\end{theorem}
The formal version and proof can be found in Appendix. Theorem~\ref{thrm:optimization_bound} states that optimizing~\eqref{eq:objective_multi_task_approx} can also give a near-optimal solution for~\eqref{eq:objective_multi_task}, as long as $\gamma$ is small. This condition holds when $p^m(\theta)$ and $p_{\bm{\alpha}}(\theta)$ do not reach a high density at the same time outside $S$. This is reasonable because overfitted parameter of main task (\ie, $\theta$ giving a high training data likelihood outside $S$) is highly random, depending on how we sample the training set, thus is unlikely to meet the optimal parameters of auxiliary tasks. In practice, we also find this approximation gives a robust result (Sec.~\ref{sec:exp_q3}).

\paragraph{Samples (P2)}
\label{para:samples}

To estimate the objective in~\eqref{eq:objective_multi_task_approx}, we need samples from $p^J(\theta) = \frac{1}{Z^\prime(\bm{\alpha})} p^m(\theta) p_{\bm{\alpha}}(\theta)$. Apparently we cannot sample from this complex distribution directly. However, we notice that $p^J$ is what we optimize in the joint-training loss~\eqref{eq:multi_task_loss}, \ie, $\mathcal{L}(\theta) \propto -\log p^J(\theta)$. To this end, we use the tool of Langevin dynamics~\cite{neal2011mcmc,welling2011bayesian} to sample from $p^J$ while optimizing the joint-loss~\eqref{eq:multi_task_loss}. Specifically, at the $t$-th step of SGD, we inject a Gaussian noise with a certain variance into the gradient step:
\begin{equation}
\label{eq:langevin_dynamics}
    \Delta \theta_t = \epsilon_t \nabla \log p^J(\theta) + \eta_t,
\end{equation}
where $\epsilon_t$ is the learning rate, and $\eta_t \sim N(0, 2\epsilon_t)$ is a Guassian noise. With the injected noise, $\theta_t$ will converge to samples from $p^J$, which can then be used to estimate~\eqref{eq:objective_multi_task_approx}. In practice, we inject noise in early epochs to sample from $p^J$ and optimize $\bm{\alpha}$, and then return to regular SGD once $\bm{\alpha}$ has converged. % Additionally, we empirically show that the noise itself will not affect the training process, thus keeping the comparison fair. Please refer to Appendix for more details.
Note that we do not anneal the learning rate as in~\cite{welling2011bayesian} because we find in practice that stochastic gradient noise is negligible compared with injected noise (see Appendix).

\paragraph{Partition Function (P3)}
\label{para:partition_function}

To estimate~\eqref{eq:objective_multi_task_approx}, we need the exact value of surrogate prior $p_{\bm{\alpha}} (\theta) = \frac{1}{Z(\bm{\alpha})} \prod_{k=1}^K p^{\alpha_k}(\mathcal{T}_{a_k} | \theta)$. Although we can easily calculate the data likelihood $p(\mathcal{T}_{a_k} | \theta)$, the partition function $Z(\bm{\alpha})$ is intractable. The same problem also occurs in model estimation~\cite{gutmann2010noise}, Bayesian inference~\cite{murray2004bayesian}, \etc. A common solution is to use score function $\nabla \log p_{\bm{\alpha}} (\theta)$ as a  substitution of $p_{\bm{\alpha}} (\theta)$ to estimate relationship with other distributions~\cite{hyvarinen2005estimation,liu2016stein,hu2018stein}. For one reason, score function can uniquely decide the distribution. It also has other nice properties. For example, the divergence defined on score functions (also known as \emph{Fisher divergence})
\begin{equation}
    F\infdivx{p}{q} = E_{\theta \sim p} \norm{\nabla \log p(\theta) - \nabla \log q(\theta)}^2_2
\end{equation}
is stronger than many other divergences including KL divergence, Hellinger distance, \etc~\cite{liu2016kernelized,hu2018stein}. Most importantly, using score function can obviate estimation of partition function which is constant \wrt $\theta$. To this end, we propose to minimize the distance between score functions instead, and our objective finally becomes
\begin{equation}
\label{eq:objective_multi_task_final}
    \min_{\bm{\alpha}} E_{\theta \sim p^J} \norm{\nabla \log p(\mathcal{T}_m | \theta) - \nabla \log p_{\bm{\alpha}} (\theta)}^2_2.
\end{equation}
Note that $\nabla \log p^m(\theta) = \nabla \log p(\mathcal{T}_m | \theta)$. In Appendix we show that under mild conditions the optimal solution for~\eqref{eq:objective_multi_task_final} is also the optimal or near-optimal $\bm{\alpha}$ for~\eqref{eq:objective_multi_task} and~\eqref{eq:objective_multi_task_approx} . We find in practice that optimizing~\eqref{eq:objective_multi_task_final} generally gives optimal weights for minimum-data learning.

\begin{algorithm}[t]
\caption{ARML (\textbf{A}uxiliary Task \textbf{R}eweighting for \textbf{M}inimum-data \textbf{L}earning)}
\label{alg}
\begin{algorithmic}

\STATE \textbf{Input:} main task data $\mathcal{T}_m$, auxiliary task data $\mathcal{T}_{a_k}$, initial parameter $\theta_0$, initial task weights $\bm{\alpha}$
\STATE \textbf{Parameters:} learning rate of $t$-th iteration $\epsilon_t$, learning rate for task weights $\beta$
%\STATE \textbf{Outputs:} Final parameter $\theta_T$
\\~\\

\FOR{iteration $t = 1$ to $T$}
    \IF{$\bm{\alpha}$ has not converged}
        \STATE $\theta_t \gets \theta_{t-1} - \epsilon_t (- \nabla \log p(\mathcal{T}_m | \theta_{t-1}) - \sum_{k=1}^K \alpha_k \nabla \log p(\mathcal{T}_{a_k} | \theta_{t-1})) + \eta_t$
        \STATE $\bm{\alpha} \gets \bm{\alpha} - \beta \nabla_{\bm{\alpha}} \norm{\nabla \log p(\mathcal{T}_m | \theta_t) - \sum_{k=1}^K \alpha_k \nabla \log p(\mathcal{T}_{a_k} | \theta_t)}^2_2$ 
        \STATE Project $\bm{\alpha}$ back into $\mathcal{A}$
    \ELSE
        \STATE $\theta_t \gets \theta_{t-1} - \epsilon_t (- \nabla \log p(\mathcal{T}_m | \theta_{t-1}) - \sum_{k=1}^K \alpha_k \nabla \log p(\mathcal{T}_{a_k} | \theta_{t-1}))$
    \ENDIF
\ENDFOR

\end{algorithmic}
\end{algorithm}

\subsection{Algorithm}
\label{sec:algorithm}

Now we present the final algorithm of auxiliary task reweighting for minimum-data learning (ARML). 
The full algorithm is shown in Alg.~\ref{alg}.
First, our objective is~\eqref{eq:objective_multi_task_final}. Until $\bm{\alpha}$ converges, we use Langevin dynamics~\eqref{eq:langevin_dynamics} to collect samples at each iteration, and then use them to estimate~\eqref{eq:objective_multi_task_final} and update $\alpha$. Additionally, we only search $\bm{\alpha}$ in an affine simplex $\mathcal{A} = \{\bm{\alpha} | \sum_k \alpha_k = K; \ \alpha_k \ge 0, \forall k\}$ to decouple task reweighting from the global weight of auxiliary tasks~\cite{chen2017gradnorm}. Please also see Appendix~\ref{appendix_tips} for details on the algorithm implementation in practice.

%%%%%%%%%%%%%%%%%%%%%%%%%%%%%%%%%%%%%%%%%%%%%%%%%%%%%%%%%%%%%%%%%%%%%%%%%%%%

\section{Experiments}
\label{sec:exp}

For experiments, we test effectiveness and robustness of ARML under various settings. This section is organized as follows. First in Sec.~\ref{sec:exp_q1}, we test whether ARML can reduce data requirement in different settings (semi-supervised learning, multi-label classification), and compare it with other reweighting methods. In Sec.~\ref{sec:exp_q2}, we study an extreme case: based on an unsupervised setting (\eg domain generalization), if a little extra labeled data is provided (\eg 1 or 5 labels per class), can ARML maximize its benefit and bring a non-trivial improvement over unsupervised baseline and other few-shot algorithms? Finally in Sec.~\ref{sec:exp_q3}, we test ARML's robustness under different levels of data scarcity and validate the rationality of approximation we made in Sec.~\ref{sec:approach}.% For details on experimental set-up and also additional results (\eg error bars), please refer to Appendix.

% In this section, we test effectiveness and robustness of ARML under various settings. Specifically, we conduct experiments to answer the following questions:
% \begin{itemize}
%     \item Can ARML reduce required amount of data compared to baseline, as well as other reweighting methods? Is the improvement consistent under different settings? (Sec. \ref{sec:exp_q1})
%     \item Can ARML still benefit main task when there is no labeled data at all (\eg unsupervised learning)? (Sec. \ref{sec:exp_q2})
%     \item Is ARML robust against different level of data scarcity? (Sec. \ref{sec:exp_q3})
% \end{itemize}

\subsection{ARML can Minimize Data Requirement}
\label{sec:exp_q1}

To get started, we show that ARML can minimize data requirement under two realist settings: semi-supervised learning and multi-label classification. 
we consider the following task reweighting methods for comparison: (i) \textbf{Uniform  (baseline)}: all weights are set to 1, (ii) \textbf{AdaLoss}~\cite{hu2019learning}: tasks are reweighted based on uncertainty, (iii) \textbf{GradNorm}~\cite{chen2017gradnorm}: balance each task's gradient norm, (iv) \textbf{CosineSim}~\cite{du2018adapting}: tasks are filtered out when having negative cosine similarity $\cos (\nabla \log p(\mathcal{T}_{a_k} | \theta), \nabla \log p(\mathcal{T}_m | \theta))$, (v) \textbf{OL\_AUX}~\cite{lin2019adaptive}: tasks have higher weights when the gradient inner product $\nabla \log p(\mathcal{T}_{a_k} | \theta)^T \nabla \log p(\mathcal{T}_m | \theta)$ is large. Besides, we also compare with grid search as an `upper bound' of ARML. Since grid search is extremely expensive, we only compare with it when the task number is small (\eg $K=2$).

\paragraph{Semi-supervised Learning (SSL)}
%In SSL, one generally trains classification as the main task on labeled data, and meanwhile updates different manually-designed loss on unlabeled data as auxiliary tasks.
In SSL, one generally trains classifier with certain percentage of labeled data as the main task, and at the same time designs different losses on unlabeled data as auxiliary tasks.
Specifically,  we use \emph{Self-supervised Semi-supervised Learning} (S4L)~\cite{zhai2019s4l} as our baseline algorithm. S4L uses self-supervised methods on unlabeled part of training data, and trains classifier on labeled data as normal. Following~\cite{zhai2019s4l}, we use two kinds of self-supervised methods: \emph{Rotation} and \emph{Exemplar-MT}. In \emph{Rotation}, we rotate each image by $[0^{\degree}, 90^{\degree}, 180^{\degree}, 270^{\degree}]$ and ask the network to predict the angle. In \emph{Exemplar-MT}, the model is trained to extract feature invariant to a wide range of image transformations. Here we use random flipping, gaussian noise~\cite{chen2020simple} and Cutout~\cite{devries2017improved} as data augmentation. During training, each image is randomly augmented, and then features of original image and augmented image are encouraged to be close. %Note that we make two differences from~\cite{zhai2019s4l}: (i) for steadier training, we use the model with time-averaged parameters~\cite{tarvainen2017mean} to extract feature of original image, (ii) To avoid over-sampling of negative samples in triplet-loss~\cite{arora2019theoretical}, we only put a loss on the cosine similarity between original feature and augmented feature.

% \begin{table}[t]\small
%   \caption{Test error of semi-supervised learning on CIFAR-10 and SVHN. From top to bottom, we report results of purely-supervised method, state-of-the-art semi-supervised methods, and S4L with different reweighting schemes, respectively. $ ^\ast$ means multiple runs are needed.}
%   \label{tab:S4L}
%   \vspace{1em}
%   \centering
%   \begin{tabular}{lcc}
%     \toprule
%      & CIFAR-10 & SVHN \\
%      & (4000 labels) & (1000 labels) \\
%     \midrule
%     Supervised & 20.26 & 12.83 \\
%     \midrule
%     $\Pi$-Model~\cite{laine2016temporal} & 16.37 & 7.19 \\
%     Mean Teacher~\cite{tarvainen2017mean} & 15.87 & 5.65 \\
%     VAT~\cite{oliver2018realistic} & 13.86 & 5.63 \\
%     VAT + EntMin~\cite{grandvalet2005semi} & \textbf{13.13} & \textbf{5.35} \\
%     Pseudo-Label~\cite{lee2013pseudo} & 17.78 & 7.62 \\
%     \midrule
%     S4L (Uniform) & 15.67 & 7.83 \\
%     S4L + AdaLoss & 21.06 & 11.53 \\
%     S4L + GradNorm & 14.07 & 7.68 \\
%     S4L + CosineSim & 15.03 & 7.02 \\
%     S4L + OL\_AUX & 16.07 & 7.82 \\
%     S4L + GridSearch^\ast & 13.76 & 6.07 \\
%     S4L + ARML (ours) & \textbf{13.68} & \textbf{5.89} \\
%     \bottomrule
%   \end{tabular}
% \end{table}

% \vspace{-1em}
\begin{minipage}{\textwidth}

\begin{minipage}{0.52\textwidth}

    \captionof{table}{Test error of semi-supervised learning on CIFAR-10 and SVHN. From top to bottom: purely-supervised method, state-of-the-art semi-supervised methods, and S4L with different reweighting schemes. $ ^\ast$ means multiple runs are needed.}
  \label{tab:S4L}
  %\vspace{1em}
  \centering
  \begin{small}
  \begin{tabular}{lcc}
    \toprule
     & CIFAR-10 & SVHN \\
     & (4000 labels) & (1000 labels) \\
    \midrule
    Supervised & 20.26 $\pm$ .38 & 12.83 $\pm$ .47 \\
    \midrule
    $\Pi$-Model~\cite{laine2016temporal} & 16.37 $\pm$ .63 & \ \ 7.19 $\pm$ .27 \\
    Mean Teacher~\cite{tarvainen2017mean} & 15.87 $\pm$ .28 & \ \ 5.65 $\pm$ .47\\
    VAT~\cite{oliver2018realistic} & 13.86 $\pm$ .27 & \ \ 5.63 $\pm$ .20 \\
    VAT + EntMin~\cite{grandvalet2005semi} & \textbf{13.13} $\pm$  .39 & \ \ \textbf{5.35} $\pm$ .19\\
    Pseudo-Label~\cite{lee2013pseudo} & 17.78 $\pm$ .57 & \ \ 7.62 $\pm$ .29 \\
    \midrule
    S4L (Uniform) & 15.67 $\pm$ .29 & \ \ 7.83 $\pm$ .33 \\
    S4L + AdaLoss & 21.06 $\pm$ .17 & 11.53 $\pm$ .39 \\
    S4L + GradNorm & 14.07 $\pm$ .44 & \ \ 7.68 $\pm$ .13 \\
    S4L + CosineSim & 15.03 $\pm$ .31 & \ \ 7.02 $\pm$ .25 \\
    S4L + OL\_AUX & 16.07 $\pm$ .51 & \ \ 7.82 $\pm$ .32 \\
    S4L + GridSearch$^\ast$ & 13.76 $\pm$ .22 & \ \ 6.07 $\pm$ .17 \\
    S4L + ARML (ours) & \textbf{13.68} $\pm$ .35 & \ \ \textbf{5.89} $\pm$ .22 \\
    \bottomrule
  \end{tabular}
  \end{small}
    
\end{minipage}\hfill
\begin{minipage}{0.45\textwidth}
    \centering
    \includegraphics[width=1.0\textwidth]{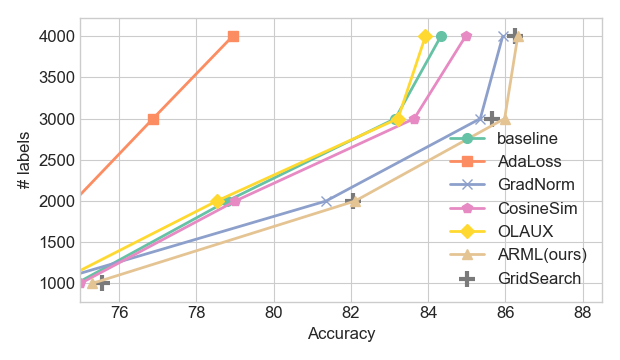} % first figure itself
    \vspace{-1.5em}
    \captionof{figure}{Amount of labeled data required to reach certain accuracy on CIFAR-10.}
    \label{fig:S4L}

    \centering
    \includegraphics[width=1.0\textwidth]{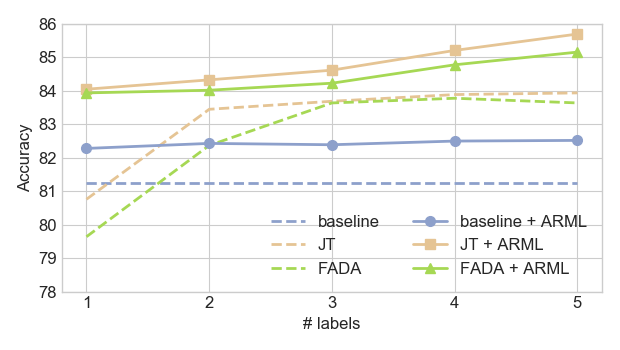} % second figure itself
    \vspace{-1.5em}
    \captionof{figure}{Accuracy of multi-source domain generalization with Art as target. }
    \label{fig:dg}
\end{minipage}

\end{minipage}
% \vspace{-0.5em}

Based on S4L, we use task reweighting to adjust the weights for different self-supervised losses. Following the literature~\cite{tarvainen2017mean,oliver2018realistic}, we test on two widely-used benchmarks: CIFAR-10~\cite{krizhevsky2009learning}  with 4000 out of 45000 images labeled, and SVHN~\cite{netzer2011reading} with 1000 out of 65932 images labeled. We report test error of S4L with different reweighting schemes in Table~\ref{tab:S4L}, along with other SSL methods. We notice that, on both datasets, with the same amount of labeled data, ARML makes a better use of the data than uniform baseline as well as other reweighting methods. Remarkably, with only one pass, ARML is able to find the optimal weights while GridSearch needs multiple runs. S4L with our ARML applied is comparable to other state-of-the-art SSL methods. 
%Note that we only try out \emph{Rotation} and \emph{Exemplar-MT}, 
%for self-supervision, %and one can simply add other tasks to further improve the performance. %Also, as an additional benefit of ARML, one does not have to worry about undermining performance when introducing new auxiliary tasks, because in the worst case, ARML can automatically filter out the detrimental task. Since exploring more auxiliary tasks is beyond the main scope of this work, we will leave it for future exploration.
%Note that we only try out \emph{Rotation} and \emph{Exemplar-MT}, and since exploring more auxiliary tasks is beyond the main scope of this work, we will leave it for future study.
Notably, we only try \emph{Rotation} and \emph{Exemplar-MT}, while exploring more auxiliary tasks could further benefit the main task and we leave it for future study.

To see whether ARML can consistently reduce data requirement, we also test the amount of data required to reach different accuracy on CIFAR-10. As shown in Fig.~\ref{fig:S4L}, with ARML applied, we only need about half of labels to reach a decent performance. This also agrees with the results of GridSearch, showing the maximum improvement from auxiliary tasks during joint training. 
%which means this is the maximum improvement auxiliary tasks could possibly provide by joint training.
%However, when data becomes extremely scarce, the advantage of ARML gets smaller, which may be caused by limited information from auxiliary tasks.

\paragraph{Multi-label Classification (MLC)}
We also test our method in MLC. We use the CelebA dataset~\cite{liu2015deep}. It contains 200K face images, each labeled with 40 binary attributes. We cast this into a MLC problem, where we randomly choose one target attribute as the main classification task, and other 39 as auxiliary tasks. To simulate our data-scarce setting, we only use 1\% labels for main task.

We test different reweighting methods and list the results in Table~\ref{tab:multi_label}. With the same amount of labeled data, ARML can help find better and more generalizable model parameters than baseline as well as other reweighting methods. This also implies that ARML has a consistent advantage even when handling a large number of tasks. For a further verification, we also check if the learned relationship between different face attributes is aligned with human's intuition. In Table~\ref{tab:multi_label_analysis}, we list the top 5 auxiliary tasks with the highest weights, and also the top 5 with the lowest weights. As we can see, ARML has automatically picked attributes describing facial hair (\eg. Mustache, Sideburns, Goatee), which coincides with the main task 5\_o\_Clock\_Shadow, another kind of facial hair. On the other hand, the tasks with low weights seem to be unrelated to the main task. This means ARML can actually learn the task relationship that matches our intuition.

\begin{table}[t]\small
    
\begin{minipage}{0.30\textwidth}
  \caption{Test error of main task on CelebA.}
  \label{tab:multi_label}
  \centering
  \begin{tabular}{lc}
    \toprule
     & Test Error\\
    \midrule
    Baseline & 6.70 $\pm$ .18 \\
    AdaLoss~\cite{hu2019learning} & 7.21 $\pm$ .11 \\
    GradNorm~\cite{chen2017gradnorm} & 6.44 $\pm$ .07 \\
    CosineSim~\cite{du2018adapting} & 6.51 $\pm$ .14 \\
    OL\_AUX~\cite{lin2019adaptive} & 6.32 $\pm$ .17 \\
    ARML (ours) & \textbf{5.97} $\pm$ .18 \\
    \bottomrule
  \end{tabular}
\end{minipage}\hfill
\begin{minipage}{0.65\textwidth}
\caption{Top 5 \textcolor{YellowGreen}{relative} / \textcolor{Maroon}{irrelative} attributes (auxiliary tasks) to the target attribute (main task) on CelebA.}
  \label{tab:multi_label_analysis}
  \centering
  \begin{tabular}{ccc}
    \toprule
    main task & most related tasks & least related tasks\\
    \midrule
    \multirow{5}{*}{5\_o\_Clock\_Shadow} & \textcolor{YellowGreen}{Mustache} & \textcolor{Maroon}{Mouth\_Slightly\_Open}\\
     & \textcolor{YellowGreen}{Bald} & \textcolor{Maroon}{Male} \\
     & \textcolor{YellowGreen}{Sideburns} & \textcolor{Maroon}{Attractive}\\
     & \textcolor{YellowGreen}{Rosy\_Cheeks} & \textcolor{Maroon}{Heavy\_Makeup}\\
     & \textcolor{YellowGreen}{Goatee} & \textcolor{Maroon}{Smiling}\\
    \bottomrule
  \end{tabular}
\end{minipage}
\end{table}

\subsection{ARML can Benefit Unsupervised Learning at Minimal Cost}
\label{sec:exp_q2}

In Sec.~\ref{sec:exp_q1}, we use ARML to reweight tasks and find a better prior for main task in order to compensate for data scarcity. Then one may naturally wonder whether this still works under situations where the main task has no labeled data at all (\eg unsupervised learning). In fact, this is a meaningful question, not only because unsupervised learning is one of the most important problems in the community, but also because using auxiliary tasks is a mainstream of unsupervised learning methods~\cite{oord2018representation,carlucci2019domain,fu2015transductive}. 
%(\eg representation learning~\cite{oord2018representation}, domain generalization~\cite{carlucci2019domain}, zero-shot learning~\cite{fu2015transductive}).
Intuitively, as long as the family of prior $p_{\bm{\alpha}} (\theta)$ is strong enough (which is determined by auxiliary tasks), we can always find a prior that gives a good model even without label information. However, if we want to use ARML to find the prior, at least \emph{some} labeled data is required to estimate the gradient for main task (Eq.~\eqref{eq:objective_multi_task_final}). 
%Then the question becomes, how much data at least does ARML need to find a proper set of weights? 
Then the question becomes, how minimum of the data does ARML need to find a proper set of weights? 
More specifically, can we use as little data as possible (\eg 1 or 5 labeled images per class) to make substantial improvement?

To answer the question, we conduct experiments in domain generalization, a well-studied unsupervised problem. In domain generalization, there is a target domain with no data (labeled or unlabeled), and multiple source domains with plenty of data. People usually train a model on source domains (auxiliary tasks) and transfer it to the target domain (main task). To use ARML, we relax the restriction a little by adding $N_m$ extra labeled images for target domain, where $N_m = 1, \cdots, 5$. This slightly relaxed setting is known as few-shot domain adaptation (FSDA) which was studied in~\cite{motiian2017few}, and we also add their FSDA results into comparison. For dataset selection, we use a common benchmark PACS~\cite{li2017deeper} which contains four distinct domains of Photo, Art, Cartoon and Sketch. We pick each one as target domain and the other three as source domains which are reweighted by our ARML.

\begin{table}[t]\small
  \caption{Results of multi-source domain generalization (w/ extra 5 labeled images per class in target domain). We list results with each of four domains as target domain. From top to down: domain generalization methods, FSDA methods and different methods equipped with ARML. JT is short for joint-training. $^\dagger$ means the results we reproduced are higher than originally reported.}
  \label{tab:dg}
  \vspace{1em}
  \centering
  \begin{tabular}{lccccc}
    \toprule
     Method & Extra label & Sketch & Art & Cartoon & Photo\\
    \midrule
    Baseline$^\dagger$ & \xmark & 75.34 & 81.25 & 77.35 & 95.93\\
    D-SAM~\cite{d2018domain} & \xmark & 77.83 & 77.33 & 72.43 & 95.30 \\
    JiGen~\cite{carlucci2019domain} & \xmark & 71.35 & 79.42 & 75.25 & 96.03 \\
    Shape-bias~\cite{asadi2019towards} & \xmark & \textbf{78.62} & \textbf{83.01} & \textbf{79.39} & \textbf{96.83} \\
    \midrule
    JT & \cmark & 78.52 & 83.94 & \textbf{81.36} & 97.01 \\
    FADA~\cite{motiian2017few} & \cmark & 79.23 & 83.64 & 79.39 & 97.07\\
    \midrule
    Baseline + ARML & \cmark & 79.35 & 82.52 & 77.30 & 95.99 \\
    JT + ARML & \cmark & \textbf{80.47} & \textbf{85.70} & 81.01 & \textbf{97.22}\\
    FADA + ARML & \cmark & 79.46 & 85.16 & 81.23 & 97.01\\
    \bottomrule
  \end{tabular}
  \
%   \vspace{-1.5em}
\end{table}

We first set $N_m = 5$ to see the results (Table~\ref{tab:dg}). Here we include both state-of-the-art domain generalization methods~\cite{asadi2019towards,carlucci2019domain,d2018domain} and FSDA methods~\cite{motiian2017few}. Since they are orthogonal to ARML, we apply ARML on both types of methods to see the relative improvement. Let us first look at domain generalization methods. Here the baseline refers to training a model on source domains (auxiliary tasks) and directly testing on target domain (main task). If we use the extra 5 labels to reweight different source domains with ARML, we can make a non-trivial improvement, especially with Sketch as target (4\% absolute improvement).
Note that in ``Baseline + ARML'', we update $\theta$ using only classification loss on source data (auxiliary loss), and the extra labeled data in the target domain are just used for reweighting the auxiliary tasks, which means the improvement completely comes from task reweighting. Additionally, joint-training (JT) and FSDA methods also use extra labeled images by adding them into classification loss. If we further use 
the extra labels for task reweighting, then we can make a further improvement and reach a state-of-the-art performance.
%If we also use them for task reweighting, then we can make a further improvement and reach a state-of-the-art performance.

We also test performance of ARML with $N_m = 1, \cdots, 5$. As an example, here we use Art as target domain. As shown in Fig.~\ref{fig:dg}, ARML is able to improve the accuracy over different domain generalization or FSDA methods. Remarkably, when $N_m = 1$, although FSDA methods are under-performed, ARML can still bring an improvement of $\sim4\%$ accuracy. 
This means ARML can benefit unsupervised domain generalization with as few as 1 labeled image per class.
%This means ARML can benefit unsupervised domain generalization with the minimal cost of 1 labeled image per class.

\subsection{ARML is Robust to Data Scarcity}
\label{sec:exp_q3}
Finally, we examine the robustness of our method. Due to the approximation made in Sec.~\ref{sec:approach}, ARML may find a suboptimal solution. For example, in the true prior approximation (\hyperref[para:true-prior]{\textbf{P1}}), we use $p(\mathcal{T}_m | \theta)$ to replace $p^\ast (\theta)$. When the size of $\mathcal{T}_m$ is large, these two should be close to each other. However, if we have less data, $p(\mathcal{T}_m | \theta)$ may also have high-value region outside $p^\ast (\theta)$ (\ie `overfitting' area), which may make the approximation inaccurate. To test the robustness of ARML, we check whether ARML can find similar task weights under different levels of data scarcity.
%To this end, we test the robustness by checking if ARML finds similar task weights under different levels of data scarcity.

\begin{figure}[ht]
% \vspace{-1em}
\centering
\begin{subfigure}[b]{0.3\linewidth}
  \centering
  % include first image
  \includegraphics[width=1.0\linewidth]{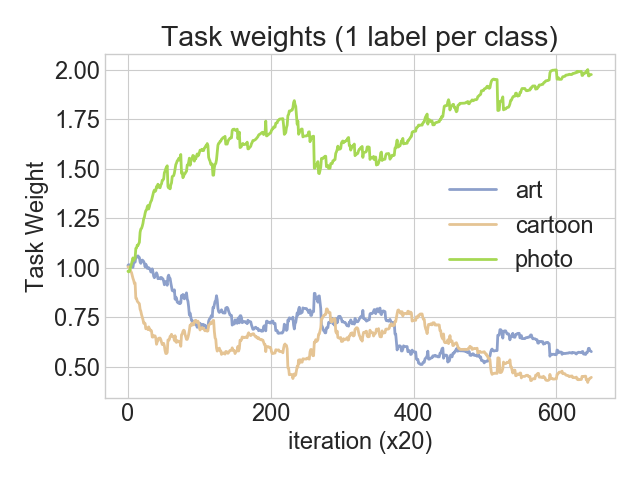}  
\end{subfigure}
\begin{subfigure}[b]{0.3\textwidth}
  \centering
  % include second image
  \includegraphics[width=1.0\linewidth]{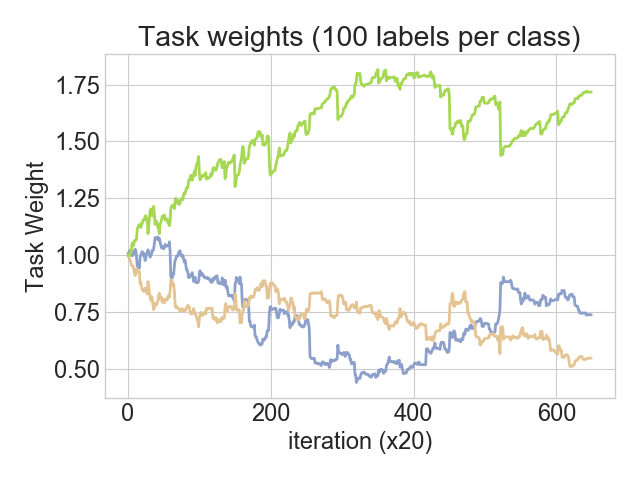}  
\end{subfigure}
\begin{subfigure}[b]{0.3\textwidth}
  \centering
  % include second image
  \includegraphics[width=1.0\linewidth]{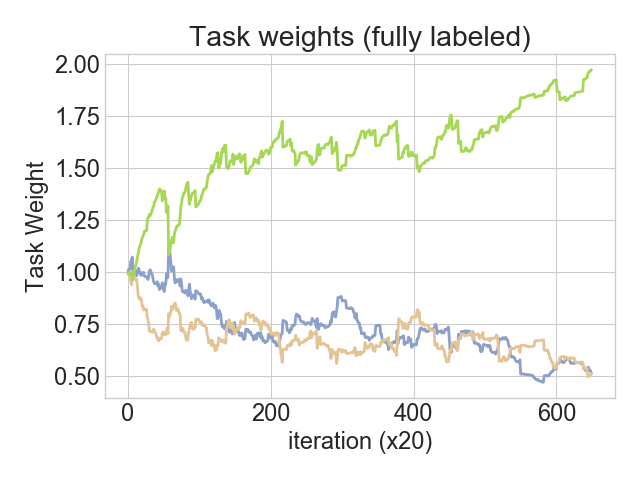}  
\end{subfigure}
\vspace{-0.5em}
\caption{Change of task weights during training under different levels of data scarcity. From left to right: one-shot, partially labeled and fully labeled.}
\label{fig:task_weights}
% \vspace{-1.5em}
\end{figure}

We conduct experiments on multi-source domain generalization with Art as target domain. We test three levels of data scarcity: few-shot (1 label per class), partly labeled (100 labels per class) and fully labeled ($\sim 300$ labels per class). We plot the change of task weights during training time in Fig.~\ref{fig:task_weights}. We can see that task weights found by ARML are barely affected by data scarcity, even in few-shot scenario. This means ARML is able to find the optimal weights even with minimal guidance, verifying the rationality of approximation in Sec.~\ref{sec:approach} and the robustness of our method.
%substantiating the soundness 

% \begin{table}[h]
%   \caption{Comparison with other reweighting methods on Semi-supervised Learning (TODO:  test different labeled ratio and \textbf{make a plot}).}
%   \centering
%   \begin{tabular}{lcc}
%     \toprule
%      & CIFAR-10 & SVHN \\
%      & (4000 labels) & (1000 labels) \\
%     \midrule
%     Baseline (S4L) & 15.67 & 7.83 \\
%     AdaLoss & 21.06 & 11.53 \\
%     GradNorm & 14.07 & 7.68 \\
%     CosineSim & 15.03 & 7.02 \\
%     OL\_AUX & 16.07 & 7.82 \\
%     ARML (ours) & \textbf{13.68} & \textbf{5.89} \\
%     \bottomrule
%   \end{tabular}
% \end{table}

%%%%%%%%%%%%%%%%%%%%%%%%%%%%%%%%%%%%%%%%%%%%%%%%%%%%%%%%%%%%%%%%%%%%%%%%%%%%

\section{Related Work}
\label{sec:related_work}

\paragraph{Additional Supervision from Auxiliary Tasks}

When there is not enough data to learn a task, it is common to introduce additional supervision from some related auxiliary tasks. For example, in semi-supervised learning, previous work has employed various kinds of manually-designed supervision on unlabeled data~\cite{zhai2019s4l,tarvainen2017mean,oliver2018realistic}. In reinforcement leaning, due to sample inefficiency, auxiliary tasks (\eg vision prediction~\cite{mirowski2016learning}, reward prediction~\cite{shelhamer2016loss}) are jointly trained to speed up convergence. In transfer learning or domain adaptation, models are trained on related domains/tasks and generalize to unseen domains~\cite{tzeng2015simultaneous,carlucci2019domain,asadi2019towards}. %The most similar setting to ours is \emph{learning using privileged information} (LUPI), where additional knowledge (\eg. meta data, additional modality) is available at training time~\cite{vapnik2009new,hoffman2016learning,sharmanska2013learning}. 
Learning using privileged information (LUPI) also employs additional knowledge (\eg. meta data, additional modality) during training time~\cite{vapnik2009new,hoffman2016learning,sharmanska2013learning}. However, LUPI does not emphasize the scarcity of training data as in our problem setting.

\paragraph{Multi-task Learning}

A highly related setting is multi-task learning (MTL). In MTL, models are trained to give high performance on different tasks simultaneously. Note that this is different from our setting because we only care about the performance on the main task. MTL is typically conducted through parameter sharing~\cite{sener2018multi,ando2005framework}, or prior sharing in a Bayesian manner~\cite{xue2007multi,heskes2000empirical,yu2005learning,bakker2003task}. 
%It is supported by a large body of evidence that parameter sharing and joint learning can achieve better generalization over learning each task independently~\cite{ando2005framework}, which also justifies our motivation.
%It is validated that 
Parameter sharing and joint learning can achieve better generalization over learning each task independently~\cite{ando2005framework}, which also motivates our work.
MTL has wide applications in areas including vision~\cite{bilen2016integrated}, language~\cite{collobert2008unified}, speech~\cite{huang2013cross}, \etc. We refer interested readers to this review~\cite{ruder2017overview}.

\paragraph{Adaptive Task Reweighting}
When learning multiple tasks, it is important to estimate the relationship between different tasks in order to balance multiple losses.
%In MTL, the losses for multiple tasks are usually balanced by estimating the relationship between different tasks with the goal of benefiting all the tasks jointly.
%When learning from multiple tasks, it is important to estimate the relationship between different tasks in order to balance multiple losses.
In MTL, this is usually realized by task clustering through a mixture prior~\cite{zhang2010convex,long2017learning,evgeniou2004regularized,bakker2003task}. However, this type of methods only screens out unrelated tasks without further differentiating related tasks. Another line of work balances multiple losses based on gradient norm~\cite{chen2017gradnorm} or uncertainty~\cite{hu2019learning,kendall2018multi}. 
%For our setting, it suffices to estimate relationship between the main task and auxiliary tasks. 
In our problem setting, the focus is changed to estimate the relationship between the main task and auxiliary tasks.
%For example, 
In~\cite{zamir2018taskonomy} task relationship is estimated based on whether the representation learned for one task can be easily reused for another task, which requires exhaustive enumeration of all the tasks. In~\cite{achille2019task2vec}, the enumeration process is vastly simplified by only considering a local landscape in the parameter space. However, a local landscape may be insufficient to represent the whole parameter distribution, especially in high dimensional cases such as deep networks. Recently, algorithms have been designed to adaptively reweight multiple tasks on the fly. For example, in~\cite{du2018adapting} tasks are filtered out when having opposite gradient direction to the main task. The most similar work to ours is~\cite{lin2019adaptive}, where the task relationship is also estimated from similarity between gradients. However, unlike our method, they use inner product as similarity metric with the goal of speeding up training.

%Earlier work either assumes that task relationship is known in advance~\cite{evgeniou2005learning}, or can be learned from a shared high-level prior~\cite{zhang2010convex,long2017learning,evgeniou2004regularized,bakker2003task}. However, these all assume that related tasks are equally important. Other work models task relationship based on the fact whether the representation learned for one task can be easily reused for another task~\cite{zamir2018taskonomy}, which requires exhausting enumeration of all tasks. In~\cite{achille2019task2vec}, the enumeration process is vastly simplified by only considering a local landscape in the parameter space. However, a local landscape may be insufficient to represent the whole parameter distribution, especially in high dimensional cases such as deep networks. Recently, algorithms have been designed to adaptively reweight multiple tasks on the fly. For example, in~\cite{du2018adapting} tasks are filtered out when having opposite gradient direction to the main task. Other methods balance tasks based on gradient norm~\cite{chen2017gradnorm} or task uncertainty~\cite{hu2019learning}. However, they do not exploit task relationship. The most similar work to ours is~\cite{lin2019adaptive}, where the task relationship is estimated from inner product between gradients. However, this algorithm is designed for speeding up training, thus behaves suboptimally in the data-scarce setting.

%%%%%%%%%%%%%%%%%%%%%%%%%%%%%%%%%%%%%%%%%%%%%%%%%%%%%%%%%%%%%%%%%%%%%%%%%%%%

\section{Conclusion}
\label{sec:conclusion}

In this work, we develop ARML, an algorithm to automatically reweight auxiliary tasks, so that the data requirement for the main task is minimized. We first formulate the weighted likelihood function of auxiliary tasks as a surrogate prior for the main task. Then the optimal weights are obtained by minimizing the divergence between the surrogate prior and the true prior. 
We design a practical algorithm by turning the optimization problem into minimizing the distance between main task gradient and auxiliary task gradients.
%Different tools or approximations are used towards designing the practical algorithm.
% We design a practical algorithm with what.....
We demonstrate its effectiveness and robustness in reducing the data requirement under various settings including the extreme case of only a few examples.

\begin{ack}
Prof. Darrell's group was supported in part by DoD, BAIR and BDD. Prof. Saenko was supported by DARPA and NSF. Prof. Hoffman was supported by DARPA. The authors also acknowledge the valuable suggestions from Colorado Reed, Dinghuai Zhang, Qi Dai, and Ziqi Pang. 
\end{ack}

\section*{Broader Impact}

In this work we focus on solving the data scarcity problem of a main task using auxiliary tasks, and propose an algorithm to automatically reweight auxiliary tasks so that the data requirement on the main task is minimized. On the bright side, this could impact the industry and society from two aspects. First, this may promote the landing of machine learning algorithms where labeled data is scarce or even unavailable, which is common in the real world. Second, our method can save the time and power resources wasted for manually tuning the auxiliary task weights with multiple runs, which is crucial in an era of environmental protection.
However, our method may lead to negative consequences if it is not used right. For example, our method may be utilized to extract information from a private dataset or system with less data under the assistance of other auxiliary tasks. Besides, our method may still fail in some situations where the auxiliary tasks are strong regularization of the main task, which may not allow the use in applications where high precision and robustness are imperative.

\bibliography{refs}

\begin{thebibliography}{10}

\bibitem{achille2019task2vec}
Alessandro Achille, Michael Lam, Rahul Tewari, Avinash Ravichandran, Subhransu
  Maji, Charless~C Fowlkes, Stefano Soatto, and Pietro Perona.
\newblock Task2vec: Task embedding for meta-learning.
\newblock In {\em Proceedings of the IEEE International Conference on Computer
  Vision}, pages 6430--6439, 2019.

\bibitem{ando2005framework}
Rie~Kubota Ando and Tong Zhang.
\newblock A framework for learning predictive structures from multiple tasks
  and unlabeled data.
\newblock {\em Journal of Machine Learning Research}, 6(Nov):1817--1853, 2005.

\bibitem{arora2019theoretical}
Sanjeev Arora, Hrishikesh Khandeparkar, Mikhail Khodak, Orestis Plevrakis, and
  Nikunj Saunshi.
\newblock A theoretical analysis of contrastive unsupervised representation
  learning.
\newblock {\em arXiv preprint arXiv:1902.09229}, 2019.

\bibitem{asadi2019towards}
Nader Asadi, Mehrdad Hosseinzadeh, and Mahdi Eftekhari.
\newblock Towards shape biased unsupervised representation learning for domain
  generalization.
\newblock {\em arXiv preprint arXiv:1909.08245}, 2019.

\bibitem{bakker2003task}
Bart Bakker and Tom Heskes.
\newblock Task clustering and gating for bayesian multitask learning.
\newblock {\em Journal of Machine Learning Research}, 4(May):83--99, 2003.

\bibitem{baxter1997bayesian}
Jonathan Baxter.
\newblock A bayesian/information theoretic model of learning to learn via
  multiple task sampling.
\newblock {\em Machine learning}, 28(1):7--39, 1997.

\bibitem{bilen2016integrated}
Hakan Bilen and Andrea Vedaldi.
\newblock Integrated perception with recurrent multi-task neural networks.
\newblock In {\em Advances in neural information processing systems}, pages
  235--243, 2016.

\bibitem{carlucci2019domain}
Fabio~M Carlucci, Antonio D'Innocente, Silvia Bucci, Barbara Caputo, and
  Tatiana Tommasi.
\newblock Domain generalization by solving jigsaw puzzles.
\newblock In {\em Proceedings of the IEEE Conference on Computer Vision and
  Pattern Recognition}, pages 2229--2238, 2019.

\bibitem{chen2020simple}
Ting Chen, Simon Kornblith, Mohammad Norouzi, and Geoffrey Hinton.
\newblock A simple framework for contrastive learning of visual
  representations.
\newblock {\em arXiv preprint arXiv:2002.05709}, 2020.

\bibitem{chen2019closer}
Wei-Yu Chen, Yen-Cheng Liu, Zsolt Kira, Yu-Chiang~Frank Wang, and Jia-Bin
  Huang.
\newblock A closer look at few-shot classification.
\newblock {\em arXiv preprint arXiv:1904.04232}, 2019.

\bibitem{chen2017gradnorm}
Zhao Chen, Vijay Badrinarayanan, Chen-Yu Lee, and Andrew Rabinovich.
\newblock Gradnorm: Gradient normalization for adaptive loss balancing in deep
  multitask networks.
\newblock {\em arXiv preprint arXiv:1711.02257}, 2017.

\bibitem{collobert2008unified}
Ronan Collobert and Jason Weston.
\newblock A unified architecture for natural language processing: Deep neural
  networks with multitask learning.
\newblock In {\em Proceedings of the 25th international conference on Machine
  learning}, pages 160--167, 2008.

\bibitem{devries2017improved}
Terrance DeVries and Graham~W Taylor.
\newblock Improved regularization of convolutional neural networks with cutout.
\newblock {\em arXiv preprint arXiv:1708.04552}, 2017.

\bibitem{du2018adapting}
Yunshu Du, Wojciech~M Czarnecki, Siddhant~M Jayakumar, Razvan Pascanu, and
  Balaji Lakshminarayanan.
\newblock Adapting auxiliary losses using gradient similarity.
\newblock {\em arXiv preprint arXiv:1812.02224}, 2018.

\bibitem{d2018domain}
Antonio D’Innocente and Barbara Caputo.
\newblock Domain generalization with domain-specific aggregation modules.
\newblock In {\em German Conference on Pattern Recognition}, pages 187--198.
  Springer, 2018.

\bibitem{eitz2012humans}
Mathias Eitz, James Hays, and Marc Alexa.
\newblock How do humans sketch objects?
\newblock {\em ACM Transactions on graphics (TOG)}, 31(4):1--10, 2012.

\bibitem{evgeniou2004regularized}
Theodoros Evgeniou and Massimiliano Pontil.
\newblock Regularized multi--task learning.
\newblock In {\em Proceedings of the tenth ACM SIGKDD international conference
  on Knowledge discovery and data mining}, pages 109--117, 2004.

\bibitem{fu2015transductive}
Yanwei Fu, Timothy~M Hospedales, Tao Xiang, and Shaogang Gong.
\newblock Transductive multi-view zero-shot learning.
\newblock {\em IEEE transactions on pattern analysis and machine intelligence},
  37(11):2332--2345, 2015.

\bibitem{grandvalet2005semi}
Yves Grandvalet and Yoshua Bengio.
\newblock Semi-supervised learning by entropy minimization.
\newblock In {\em Advances in neural information processing systems}, pages
  529--536, 2005.

\bibitem{griffin2007caltech}
Gregory Griffin, Alex Holub, and Pietro Perona.
\newblock Caltech-256 object category dataset.
\newblock 2007.

\bibitem{gutmann2010noise}
Michael Gutmann and Aapo Hyv{\"a}rinen.
\newblock Noise-contrastive estimation: A new estimation principle for
  unnormalized statistical models.
\newblock In {\em Proceedings of the Thirteenth International Conference on
  Artificial Intelligence and Statistics}, pages 297--304, 2010.

\bibitem{he2016deep}
Kaiming He, Xiangyu Zhang, Shaoqing Ren, and Jian Sun.
\newblock Deep residual learning for image recognition.
\newblock In {\em Proceedings of the IEEE conference on computer vision and
  pattern recognition}, pages 770--778, 2016.

\bibitem{heskes2000empirical}
TM~Heskes.
\newblock Empirical bayes for learning to learn.
\newblock In {\em Proceedings of the 17th international conference on Machine
  learning}, pages 364--367, 2000.

\bibitem{hoffman2016learning}
Judy Hoffman, Saurabh Gupta, and Trevor Darrell.
\newblock Learning with side information through modality hallucination.
\newblock In {\em Proceedings of the IEEE Conference on Computer Vision and
  Pattern Recognition}, pages 826--834, 2016.

\bibitem{hu2019learning}
Hanzhang Hu, Debadeepta Dey, Martial Hebert, and J~Andrew Bagnell.
\newblock Learning anytime predictions in neural networks via adaptive loss
  balancing.
\newblock In {\em Proceedings of the AAAI Conference on Artificial
  Intelligence}, volume~33, pages 3812--3821, 2019.

\bibitem{hu2018stein}
Tianyang Hu, Zixiang Chen, Hanxi Sun, Jincheng Bai, Mao Ye, and Guang Cheng.
\newblock Stein neural sampler.
\newblock {\em arXiv preprint arXiv:1810.03545}, 2018.

\bibitem{huang2013cross}
Jui-Ting Huang, Jinyu Li, Dong Yu, Li~Deng, and Yifan Gong.
\newblock Cross-language knowledge transfer using multilingual deep neural
  network with shared hidden layers.
\newblock In {\em 2013 IEEE International Conference on Acoustics, Speech and
  Signal Processing}, pages 7304--7308. IEEE, 2013.

\bibitem{hyvarinen2005estimation}
Aapo Hyv{\"a}rinen.
\newblock Estimation of non-normalized statistical models by score matching.
\newblock {\em Journal of Machine Learning Research}, 6(Apr):695--709, 2005.

\bibitem{ioffe2015batch}
Sergey Ioffe and Christian Szegedy.
\newblock Batch normalization: Accelerating deep network training by reducing
  internal covariate shift.
\newblock {\em arXiv preprint arXiv:1502.03167}, 2015.

\bibitem{jaderberg2016reinforcement}
Max Jaderberg, Volodymyr Mnih, Wojciech~Marian Czarnecki, Tom Schaul, Joel~Z
  Leibo, David Silver, and Koray Kavukcuoglu.
\newblock Reinforcement learning with unsupervised auxiliary tasks.
\newblock {\em arXiv preprint arXiv:1611.05397}, 2016.

\bibitem{kendall2018multi}
Alex Kendall, Yarin Gal, and Roberto Cipolla.
\newblock Multi-task learning using uncertainty to weigh losses for scene
  geometry and semantics.
\newblock In {\em Proceedings of the IEEE conference on computer vision and
  pattern recognition}, pages 7482--7491, 2018.

\bibitem{kingma2014adam}
Diederik~P Kingma and Jimmy Ba.
\newblock Adam: A method for stochastic optimization.
\newblock {\em arXiv preprint arXiv:1412.6980}, 2014.

\bibitem{krizhevsky2009learning}
Alex Krizhevsky, Geoffrey Hinton, et~al.
\newblock Learning multiple layers of features from tiny images.
\newblock 2009.

\bibitem{laine2016temporal}
Samuli Laine and Timo Aila.
\newblock Temporal ensembling for semi-supervised learning.
\newblock {\em arXiv preprint arXiv:1610.02242}, 2016.

\bibitem{lee2013pseudo}
Dong-Hyun Lee.
\newblock Pseudo-label: The simple and efficient semi-supervised learning
  method for deep neural networks.
\newblock In {\em Workshop on challenges in representation learning, ICML},
  volume~3, page~2, 2013.

\bibitem{li2017deeper}
Da~Li, Yongxin Yang, Yi-Zhe Song, and Timothy~M Hospedales.
\newblock Deeper, broader and artier domain generalization.
\newblock In {\em Proceedings of the IEEE international conference on computer
  vision}, pages 5542--5550, 2017.

\bibitem{lin2019adaptive}
Xingyu Lin, Harjatin Baweja, George Kantor, and David Held.
\newblock Adaptive auxiliary task weighting for reinforcement learning.
\newblock In {\em Advances in Neural Information Processing Systems}, pages
  4773--4784, 2019.

\bibitem{liu2016kernelized}
Qiang Liu, Jason Lee, and Michael Jordan.
\newblock A kernelized stein discrepancy for goodness-of-fit tests.
\newblock In {\em International conference on machine learning}, pages
  276--284, 2016.

\bibitem{liu2016stein}
Qiang Liu and Dilin Wang.
\newblock Stein variational gradient descent: A general purpose bayesian
  inference algorithm.
\newblock In {\em Advances in neural information processing systems}, pages
  2378--2386, 2016.

\bibitem{liu2015deep}
Ziwei Liu, Ping Luo, Xiaogang Wang, and Xiaoou Tang.
\newblock Deep learning face attributes in the wild.
\newblock In {\em Proceedings of the IEEE international conference on computer
  vision}, pages 3730--3738, 2015.

\bibitem{long2017learning}
Mingsheng Long, Zhangjie Cao, Jianmin Wang, and S~Yu Philip.
\newblock Learning multiple tasks with multilinear relationship networks.
\newblock In {\em Advances in neural information processing systems}, pages
  1594--1603, 2017.

\bibitem{maas2013rectifier}
Andrew~L Maas, Awni~Y Hannun, and Andrew~Y Ng.
\newblock Rectifier nonlinearities improve neural network acoustic models.
\newblock In {\em Proc. icml}, volume~30, page~3, 2013.

\bibitem{mirowski2016learning}
Piotr Mirowski, Razvan Pascanu, Fabio Viola, Hubert Soyer, Andrew~J Ballard,
  Andrea Banino, Misha Denil, Ross Goroshin, Laurent Sifre, Koray Kavukcuoglu,
  et~al.
\newblock Learning to navigate in complex environments.
\newblock {\em arXiv preprint arXiv:1611.03673}, 2016.

\bibitem{motiian2017few}
Saeid Motiian, Quinn Jones, Seyed Iranmanesh, and Gianfranco Doretto.
\newblock Few-shot adversarial domain adaptation.
\newblock In {\em Advances in Neural Information Processing Systems}, pages
  6670--6680, 2017.

\bibitem{murray2004bayesian}
Iain Murray and Zoubin Ghahramani.
\newblock Bayesian learning in undirected graphical models: approximate mcmc
  algorithms.
\newblock In {\em Proceedings of the 20th conference on Uncertainty in
  artificial intelligence}, pages 392--399. AUAI Press, 2004.

\bibitem{neal2011mcmc}
Radford~M Neal et~al.
\newblock Mcmc using hamiltonian dynamics.
\newblock {\em Handbook of markov chain monte carlo}, 2(11):2, 2011.

\bibitem{netzer2011reading}
Yuval Netzer, Tao Wang, Adam Coates, Alessandro Bissacco, Bo~Wu, and Andrew~Y
  Ng.
\newblock Reading digits in natural images with unsupervised feature learning.
\newblock 2011.

\bibitem{oliver2018realistic}
Avital Oliver, Augustus Odena, Colin~A Raffel, Ekin~Dogus Cubuk, and Ian
  Goodfellow.
\newblock Realistic evaluation of deep semi-supervised learning algorithms.
\newblock In {\em Advances in Neural Information Processing Systems}, pages
  3235--3246, 2018.

\bibitem{oord2018representation}
Aaron van~den Oord, Yazhe Li, and Oriol Vinyals.
\newblock Representation learning with contrastive predictive coding.
\newblock {\em arXiv preprint arXiv:1807.03748}, 2018.

\bibitem{paszke2019pytorch}
Adam Paszke, Sam Gross, Francisco Massa, Adam Lerer, James Bradbury, Gregory
  Chanan, Trevor Killeen, Zeming Lin, Natalia Gimelshein, Luca Antiga, et~al.
\newblock Pytorch: An imperative style, high-performance deep learning library.
\newblock In {\em Advances in Neural Information Processing Systems}, pages
  8024--8035, 2019.

\bibitem{ruder2017overview}
Sebastian Ruder.
\newblock An overview of multi-task learning in deep neural networks.
\newblock {\em arXiv preprint arXiv:1706.05098}, 2017.

\bibitem{sangkloy2016sketchy}
Patsorn Sangkloy, Nathan Burnell, Cusuh Ham, and James Hays.
\newblock The sketchy database: learning to retrieve badly drawn bunnies.
\newblock {\em ACM Transactions on Graphics (TOG)}, 35(4):1--12, 2016.

\bibitem{sener2018multi}
Ozan Sener and Vladlen Koltun.
\newblock Multi-task learning as multi-objective optimization.
\newblock In {\em Advances in Neural Information Processing Systems}, pages
  527--538, 2018.

\bibitem{sharmanska2013learning}
Viktoriia Sharmanska, Novi Quadrianto, and Christoph~H Lampert.
\newblock Learning to rank using privileged information.
\newblock In {\em Proceedings of the IEEE International Conference on Computer
  Vision}, pages 825--832, 2013.

\bibitem{shelhamer2016loss}
Evan Shelhamer, Parsa Mahmoudieh, Max Argus, and Trevor Darrell.
\newblock Loss is its own reward: Self-supervision for reinforcement learning.
\newblock {\em arXiv preprint arXiv:1612.07307}, 2016.

\bibitem{tarvainen2017mean}
Antti Tarvainen and Harri Valpola.
\newblock Mean teachers are better role models: Weight-averaged consistency
  targets improve semi-supervised deep learning results.
\newblock In {\em Advances in neural information processing systems}, pages
  1195--1204, 2017.

\bibitem{tzeng2015simultaneous}
Eric Tzeng, Judy Hoffman, Trevor Darrell, and Kate Saenko.
\newblock Simultaneous deep transfer across domains and tasks.
\newblock In {\em Proceedings of the IEEE International Conference on Computer
  Vision}, pages 4068--4076, 2015.

\bibitem{tzeng2017adversarial}
Eric Tzeng, Judy Hoffman, Kate Saenko, and Trevor Darrell.
\newblock Adversarial discriminative domain adaptation.
\newblock In {\em Proceedings of the IEEE Conference on Computer Vision and
  Pattern Recognition}, pages 7167--7176, 2017.

\bibitem{vapnik2009new}
Vladimir Vapnik and Akshay Vashist.
\newblock A new learning paradigm: Learning using privileged information.
\newblock {\em Neural networks}, 22(5-6):544--557, 2009.

\bibitem{welling2011bayesian}
Max Welling and Yee~W Teh.
\newblock Bayesian learning via stochastic gradient langevin dynamics.
\newblock In {\em Proceedings of the 28th international conference on machine
  learning (ICML-11)}, pages 681--688, 2011.

\bibitem{xue2007multi}
Ya~Xue, Xuejun Liao, Lawrence Carin, and Balaji Krishnapuram.
\newblock Multi-task learning for classification with dirichlet process priors.
\newblock {\em Journal of Machine Learning Research}, 8(Jan):35--63, 2007.

\bibitem{yu2005learning}
Kai Yu, Volker Tresp, and Anton Schwaighofer.
\newblock Learning gaussian processes from multiple tasks.
\newblock In {\em Proceedings of the 22nd international conference on Machine
  learning}, pages 1012--1019, 2005.

\bibitem{zamir2018taskonomy}
Amir~R Zamir, Alexander Sax, William Shen, Leonidas~J Guibas, Jitendra Malik,
  and Silvio Savarese.
\newblock Taskonomy: Disentangling task transfer learning.
\newblock In {\em Proceedings of the IEEE Conference on Computer Vision and
  Pattern Recognition}, pages 3712--3722, 2018.

\bibitem{zhai2019s4l}
Xiaohua Zhai, Avital Oliver, Alexander Kolesnikov, and Lucas Beyer.
\newblock S4l: Self-supervised semi-supervised learning.
\newblock In {\em Proceedings of the IEEE international conference on computer
  vision}, pages 1476--1485, 2019.

\bibitem{zhang2010convex}
Yu~Zhang and Dit-Yan Yeung.
\newblock A convex formulation for learning task relationships in multi-task
  learning.
\newblock In {\em Proceedings of the Twenty-Sixth Conference on Uncertainty in
  Artificial Intelligence}, pages 733--742, 2010.

\end{thebibliography}
\bibliographystyle{plain}

\appendix
\section{Additional Discussion on ARML}

In this section we add more discussion on validity and soundness of ARML, especially on the three problems (\textbf{True Prior (P1)}, \textbf{Samples (P2)}, \textbf{Partition Function (P3)}), and how we resolve them (Sec. 3.3). 

%%%%%%%%%%%%%%%%%%%%%%%%%%%%%%%%%%%%%%%%%%%%%%%%%%%%%%%%%%%%%%%%%%%%

\subsection{Full Version and Proof of Theorem 1 (P1)}

In \textbf{True Prior (P1)} (Sec. 3.3) we use 
\begin{equation}\small
\label{eq:objective_approx}
    \min_{\bm{\alpha}} E_{\theta \sim p^J} \log \frac{p^m(\theta)}{p_{\bm{\alpha}} (\theta)}
\end{equation}
as a surrogate objective for the original optimization problem
\begin{equation}\small
\label{eq:objective}
    \min_{\bm{\alpha}} D_{\mathrm{KL}}\infdivx{p^\ast (\theta)}{p_{\bm{\alpha}}(\theta)}.
\end{equation}
In this section, we will first intuitively explain why optimizing~\eqref{eq:objective_approx} can end up with a near-optimal solution for~\eqref{eq:objective}, and what assumptions do we need to make. Then we will give the full version of Theorem 1 and also the proof.

Let $f(\a) = E_{\theta \sim p^J} \log \frac{p^m(\theta)}{p_{\a}  (\theta)} = \frac{1}{Z(\a)} \int p^m(\theta) p_{\a}(\theta) \log \frac{p^m(\theta)}{p_{\a}  (\theta)} d\theta$ be the optimization objective in~\eqref{eq:objective_approx}, where $p^J(\theta) = \frac{p^m(\theta) p_{\a}(\theta)}{Z(\a)}$ and $Z(\a) = \int p^m(\theta) p_{\a}(\theta) d\theta$ is the normalization term. Assume $p^\ast(\theta)$ has a compact support set $S$. Then we can write $f(\a)$ as
\begin{equation}\small
\begin{split}
    f(\a) &= 
    \frac{1}{Z(\a)} \int_{\theta \in S} p^m(\theta) p_{\a}(\theta) \log \frac{p^m(\theta)}{p_{\a}  (\theta)} d\theta + 
    \frac{1}{Z(\a)} \int_{\theta \notin S} p^m(\theta) p_{\a}(\theta) \log \frac{p^m(\theta)}{p_{\a}  (\theta)} d\theta \\
    &= 
    \frac{Z(S; \a)}{Z(S; \a) + Z(\Sbar; \a)} \int_{\theta \in S} \frac{p^m(\theta) p_{\a}(\theta)}{Z(S; \a)} \log \frac{p^m(\theta)}{p_{\a}  (\theta)} d\theta \\
    & \quad \ + \frac{Z(\Sbar; \a)}{Z(S; \a) + Z(\Sbar; \a)} \int_{\theta \notin S} \frac{p^m(\theta) p_{\a}(\theta)}{Z(\Sbar; \a)} \log \frac{p^m(\theta)}{p_{\a}  (\theta)} d\theta \\
    &= f(\a; S) + f(\a; \Sbar),
\end{split}
\end{equation}
where we denote the first and second term by $f(\a; S)$ and $f(\a; \Sbar)$ respectively, $Z(S; \a) = \int_{\theta \in S} p^m(\theta) p_{\a}(\theta) d\theta$ and $Z(\Sbar; \a) = \int_{\theta \notin S} p^m(\theta) p_{\a}(\theta) d\theta$ are the normalization terms inside and outside $S$.

To build the connection between the surrogate objective $f(\a)$ and the original objective $\KLa := D_{\mathrm{KL}}\infdivx{p^\ast (\theta)}{p_{\bm{\alpha}}(\theta)}$, we make the following assumption,
\begin{assumption}
\label{assumption:1}
The support set $S$ is small so that $p_{\a} (\theta)$ and $p^m(\theta)$ are constants inside S, and $p^\ast(\theta)$ is uniform in $S$.
\end{assumption}
This assumption is reasonable when $S$ is really informative, which we assume is the case for the true prior $p^\ast(\theta)$~\cite{baxter1997bayesian}. With this assumption, we have
\begin{equation}\small
    \KLa = \Sint p^\ast(\theta) \log \frac{p^\ast(\theta)}{p_{\a} (\theta)} d\theta
    = \log \frac{p^\ast(\theta^\ast)}{p_{\a} (\theta^\ast)} \cdot \Sint p^\ast(\theta) d\theta = \log \frac{p^\ast(\theta^\ast)}{p_{\a} (\theta^\ast)},
\end{equation}
where $\theta^\ast \in S$ is the optimal parameter. We can also write $f(\a; S)$ as
\begin{equation}\small
\begin{split}
    f(\a; S) &= \frac{Z(S; \a)}{Z(S; \a) + Z(\Sbar; \a)} \int_{\theta \in S} \frac{p^m(\theta) p_{\a}(\theta)}{Z(S; \a)} \log \frac{p^m(\theta)}{p_{\a}  (\theta)} d\theta \\
    &= \frac{Z(S; \a)}{Z(S; \a) + Z(\Sbar; \a)}  \log \frac{p^m(\theta^\ast)}{p_{\a}  (\theta^\ast)} \cdot 
    \Sint \frac{p^m(\theta) p_{\a}(\theta)}{Z(S; \a)} d\theta \\
    &= \frac{Z(S; \a)}{Z(S; \a) + Z(\Sbar; \a)}  \log \frac{p^m(\theta^\ast)}{p_{\a}  (\theta^\ast)} \\
    &= \frac{Z(S; \a)}{Z(S; \a) + Z(\Sbar; \a)} (\log \frac{p^\ast(\theta^\ast)}{p_{\a}  (\theta^\ast)} + 
    \log \frac{p^m(\theta^\ast)}{p^\ast  (\theta^\ast)}) \\
    &= \frac{Z(S; \a)}{Z(S; \a) + Z(\Sbar; \a)} (\KLa + 
    C_1),
\end{split}
\end{equation}
where $C_1 = \log \frac{p^m(\theta^\ast)}{p^\ast  (\theta^\ast)}$ is a constant invariant to $\a$. Since $p^m(\theta)$ also covers other ``overfitting'' area other than $S$, we can assume that $p^\ast(\theta^\ast) \geq p^m(\theta^\ast)$, which gives $C_1 \leq 0$. Furthermore, we can notice that
\begin{equation}\small
    Z(S; \a) = \Sint p^m(\theta) p_{\a}(\theta) d\theta 
    = \Sint \frac{p^m(\theta) p_{\a}(\theta)}{p^\ast (\theta)} p^\ast (\theta) d\theta 
    = \frac{p^m(\theta^\ast) p_{\a}(\theta^\ast)}{p^\ast (\theta^\ast)} = C_2 e^{-\KLa},
\end{equation}
where $C_2 = p^m(\theta^\ast)$ is a constant invariant to $\a$. Then we can write $f(\a; S)$ as 
\begin{equation}\small
    f(\a; S) = \frac{C_2 e^{-\KLa}}{C_2 e^{-\KLa} + Z(\Sbar; \a)} (\KLa + 
    C_1).
\end{equation}
In this way, we build the connection between the surrogate objective $f(\a)$ and the original objective $\KLa$.

Now we give an intuitive explanation for why optimizing $f(\a)$ gives a small $\KLa$ as well. We can write $f(\a)$ as 
\begin{equation}\small
\begin{split}
    f(\a) &= f(\a; S) + f(\a; \Sbar) \\
    &= \frac{C_2 e^{-\KLa}}{C_2 e^{-\KLa} + Z(\Sbar; \a)} (\KLa + C_1) 
    + \frac{Z(\Sbar; \a)}{C_2 e^{-\KLa} + Z(\Sbar; \a)} \Sbarint \frac{p^m(\theta) p_{\a}(\theta)}{Z(\Sbar; \a)} \log \frac{p^m(\theta)}{p_{\a}  (\theta)} d\theta.
\end{split}
\end{equation}
As one can notice, $f(\a)$ not only depends on $\KLa$, but also on $Z(\Sbar; \a)$ and the integral $\Sbarint \frac{p^m(\theta) p_{\a}(\theta)}{Z(\Sbar; \a)} \log \frac{p^m(\theta)}{p_{\a}  (\theta)} d\theta$. First we remove the dependency on the integral by taking its lower bound and upper bound. Concretely, with Jensen's inequality, we have
\begin{equation}\small
    \Sbarint \frac{p^m(\theta) p_{\a}(\theta)}{Z(\Sbar; \a)} \log \frac{p^m(\theta)}{p_{\a}  (\theta)} d\theta 
    \leq 
    \log \frac{\Sbarint (p^m(\theta))^2 d\theta}{Z(\Sbar; \a)} = \log \frac{C_3}{Z(\Sbar; \a)},
\end{equation}
where $C_3 = \Sbarint (p^m(\theta))^2 d\theta$ is a constant invariant to $\a$. Likewise, we have
\begin{equation}\small
\begin{split}
    \Sbarint \frac{p^m(\theta) p_{\a}(\theta)}{Z(\Sbar; \a)} \log \frac{p^m(\theta)}{p_{\a}  (\theta)} d\theta 
    &= 
    \Sbarint - \frac{p^m(\theta) p_{\a}(\theta)}{Z(\Sbar; \a)} \log \frac{p_{\a} (\theta)}{p^m(\theta)} d\theta \\
    &\geq 
    - \log \frac{\Sbarint (p_{\a}(\theta))^2 d\theta}{Z(\Sbar; \a)} \\
    &\geq 
    - \log \frac{C_4}{Z(\Sbar; \a)},
\end{split}
\end{equation}
where $C_4 = \max_{\a} \Sbarint (p_{\a}(\theta))^2 d\theta$ is a constant invariant to $\a$. In this way, we get the lower bound and upper bound for $f(\a)$:
\begin{equation}\small
\begin{split}
    f(\a) &\geq f_l (\a) = \frac{C_2 e^{-\KLa}}{C_2 e^{-\KLa} + Z(\Sbar; \a)} (\KLa + C_1) - \frac{Z(\Sbar; \a)}{C_2 e^{-\KLa} + Z(\Sbar; \a)} \log \frac{C_4}{Z(\Sbar; \a)}, \\
    f(\a) &\leq f_u (\a) =
    \frac{C_2 e^{-\KLa}}{C_2 e^{-\KLa} + Z(\Sbar; \a)} (\KLa + C_1) + \frac{Z(\Sbar; \a)}{C_2 e^{-\KLa} + Z(\Sbar; \a)} \log \frac{C_3}{Z(\Sbar; \a)}.
\end{split}
\end{equation}

We plot $f_l$ and $f_u$ as functions of $e^{-\KLa}$ in Fig.~\ref{fig:upper&lower} (here we assume $Z(\Sbar; \a)$ is constant \wrt $\a$ for brevity). $f(\a)$ lies between the upper bound (golden line) and the lower bound (blue line).

\begin{figure}[t]
\vspace{-1em}
\begin{subfigure}[b]{0.48\linewidth}
  \centering
  % include first image
  \includegraphics[width=1.0\linewidth]{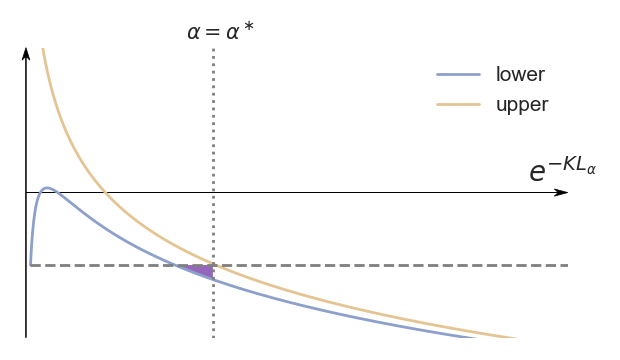}  
  \caption{$e^{-\KLastar}$ is large.}
  \label{fig:upper&lower_a}
\end{subfigure}
\begin{subfigure}[b]{0.48\textwidth}
  \centering
  % include second image
  \includegraphics[width=1.0\linewidth]{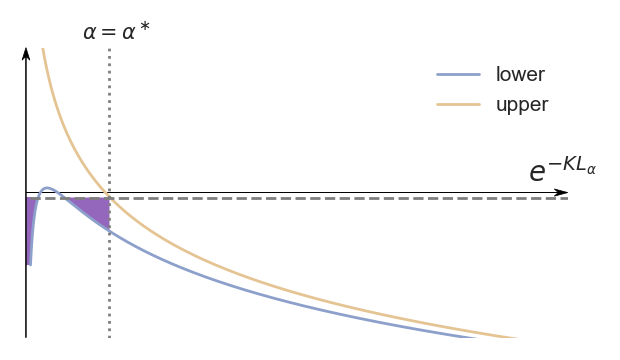}  
  \caption{$e^{-\KLastar}$ is small.}
  \label{fig:upper&lower_b}
\end{subfigure}
\caption{$f(e^{-\KLa})$'s upper bound $f_u(e^{-\KLa})$ (golden line) and lower bound $f_l(e^{-\KLa})$ (blue line). $\astar = \argmax_{\a} (e^{-\KLa}) = \argmin_{\a} \KLa$ denotes the largest $e^{-\KLa}$ we could possibly reach. Shaded region denotes where ($e^{-\KLahat}$, $f(e^{-\KLahat})$) could possibly be.}
\label{fig:upper&lower}
\end{figure}

Our goal is to find the optimal $\astar$ that minimizes $\KLa$, \ie, $\astar = \argmin_{\a} \KLa = \argmax_{\a} e^{-\KLa}$. By optimizing $f(\a)$, we end up with a suboptimal $\ahat = \argmin_{\a} f(\a)$. Ideally, we hope that $\KLahat$ is close to $\KLastar$, which means when we minimize $f(\ahat)$, we can also get a large $e^{-\KLahat}$. This is the case when $e^{-\KLastar}$ is large (see Fig.~\ref{fig:upper&lower_a}). When $e^{-\KLastar}$ is large, the upper bound $f_u$ and the lower bound $f_l$ are close to each other around $e^{-\KLastar}$ (this is the case when $Z(\Sbar; \a)$ is small). Since we have
\begin{equation}\small
\label{eq:lie_in_region}
    f_l(e^{-\KLahat}) \leq f(e^{-\KLahat}) \leq f(e^{-\KLastar}) \leq f_u(e^{-\KLastar}),
\end{equation}
we can assert that ($e^{-\KLahat}$, $f(e^{-\KLahat})$) lies in the shaded region, because if $e^{-\KLahat}$ is on the left side of the region, we have $f(e^{-\KLahat}) \geq f_u(e^{-\KLastar})$ which is contradictary to~\eqref{eq:lie_in_region}, and if $e^{-\KLahat}$ cannot be on the right side of the region because $e^{-\KLastar}$ is the furthest we can go. Since the shaded region is small, $\KLahat$ is thus close to the optimal solution $\KLastar$.

Unfortunately, this may not hold anymore when $e^{-\KLastar}$ is small (see Fig.~\ref{fig:upper&lower_b}). This is because $f_l$ will reach a local minima when $e^{-\KLa} \to 0$. If $e^{-\KLastar}$ is not large enough, it may be higher than $\lim_{e^{-\KLa} \to 0} f_l(e^{-\KLa})$, which means the shaded region near y-axis is also included. In this region $f(\a)$ could be really small (which is the goal when optimizing the surrogate objective $f(\a)$), but $\KLa$ could be extremely large.

To avoid this situation, we only have to assume that 
\begin{equation}\small
    f_u(e^{-\KLastar}) \leq \lim_{e^{-\KLa} \to 0} f_l(e^{-\KLa}) = - \log \frac{C_4}{Z(\Sbar; \a)},
\end{equation}
or if we denote $\gamma_1 = \min_{\a} Z(\Sbar; \a)$ and $\gamma_2 = \max_{\a} Z(\Sbar; \a)$, then we only need the following assumption:
\begin{assumption}
\label{assumption:2}
The optimal $\KLastar$ is small so that $f_u(e^{-\KLastar}) \leq -\log \frac{C_4}{\gamma_1}$.
\end{assumption}
This assumption holds as long as there is at least one task that is related to the main task (having a small $\KLa$), which is reasonable because if all the tasks are unrelated, then reweighing is also meaningless. See the remark below for more discussion on the validity of the assumption.

Now we give the formal version of the theorem:
\begin{theorem}
\label{thrm:full}
(formal version) With Assumption \hyperref[assumption:1]{1}, \hyperref[assumption:2]{2}, if $\gamma_2 \leq \min (\frac{C_3}{e}, \frac{C_4}{e})$, then we have
\begin{equation}\small
    \KLahat
    \leq 
    \KLastar + \frac{2\gamma_2^2}{C} \log \frac{C^\prime}{\gamma_2}
\end{equation}
\end{theorem}

\begin{proof}

From Assumption~\ref{assumption:2} we have
\begin{equation}\small
\label{eq:a15}
    \frac{C_2 e^{-\KLastar}}{C_2 e^{-\KLastar} + Z(\Sbar; \astar)} (\KLastar + C_1) + \frac{Z(\Sbar; \astar)}{C_2 e^{-\KLastar} + Z(\Sbar; \astar)} \log \frac{C_3}{Z(\Sbar; \astar)} \leq -\log \frac{C_4}{\gamma_1}.
\end{equation}
Since $\gamma_2 \leq  C_3$ and $\gamma_2 \leq C_4$, we have $\log \frac{C_3}{Z(\Sbar; \astar)} \geq \log \frac{C_3}{\gamma_2} \geq 0$, and $-\log \frac{C_4}{\gamma_1} \leq -\log \frac{C_4}{\gamma_2} \leq 0$. Then leaves us $\KLastar + C_1 \leq 0$ in order to make~\eqref{eq:a15} satisfied. Then we can relax~\eqref{eq:a15} into 
\begin{equation}\small
    \KLastar + C_1 \leq -\log \frac{C_4}{\gamma_1},
\end{equation}
which gives
\begin{equation}\small
\label{eq:bound_1}
    C_2 e^{-\KLastar} \geq \frac{C_5}{\gamma_1},
\end{equation}
where $C_5 = C_4 C_2 e^{C_1}$. This bounds the value of $\KLastar$.

Moreover, from~\eqref{eq:lie_in_region} and Assumption~\ref{assumption:2} we have
\begin{equation}\small
    f_l(e^{-\KLahat}) \leq f_u(e^{-\KLastar}) \leq 
    -\log \frac{C_4}{\gamma_1},
\end{equation}
which gives
\begin{equation}\small
\label{eq:a19}
    f_l(e^{-\KLahat}) = 
    \frac{C_2 e^{-\KLahat}}{C_2 e^{-\KLahat} + Z(\Sbar; \ahat)} (\KLahat + C_1) - \frac{Z(\Sbar; \ahat)}{C_2 e^{-\KLahat} + Z(\Sbar; \ahat)} \log \frac{C_4}{Z(\Sbar; \ahat)}
    \leq 
    -\log \frac{C_4}{\gamma_1}.
\end{equation}
Since $Z(\Sbar; \ahat) \geq \gamma_1$, we can relax~\eqref{eq:a19} into
\begin{equation}\small
    \frac{C_2 e^{-\KLahat}}{C_2 e^{-\KLahat} + Z(\Sbar; \ahat)} (\KLahat + C_1) - \frac{Z(\Sbar; \ahat)}{C_2 e^{-\KLahat} + Z(\Sbar; \ahat)} \log \frac{C_4}{Z(\Sbar; \ahat)}
    \leq 
    -\log \frac{C_4}{Z(\Sbar; \ahat)},
\end{equation}
which can be simplified into
\begin{equation}\small
    \KLahat + C_1 \leq -\log \frac{C_4}{Z(\Sbar; \ahat)}
    \leq -\log \frac{C_4}{\gamma_2},
\end{equation}
which means
\begin{equation}\small
\label{eq:bound_2}
    C_2 e^{-\KLahat} \geq \frac{C_5}{\gamma_2}.
\end{equation}
This bounds the value of $\KLahat$.

Now we build the connection between $\KLahat$ and $\KLastar$. Since $f_l(e^{-\KLahat}) \leq f_u(e^{-\KLastar})$, we have
\begin{equation}\small
\label{eq:a23}
\begin{split}
    & \quad \ \ \frac{C_2 e^{-\KLahat}}{C_2 e^{-\KLahat} + Z(\Sbar; \ahat)} (\KLahat + C_1) - \frac{Z(\Sbar; \ahat)}{C_2 e^{-\KLahat} + Z(\Sbar; \ahat)} \log \frac{C_4}{Z(\Sbar; \ahat)} \\
    & \leq 
    \frac{C_2 e^{-\KLastar}}{C_2 e^{-\KLastar} + Z(\Sbar; \astar)} (\KLastar + C_1) + \frac{Z(\Sbar; \astar)}{C_2 e^{-\KLastar} + Z(\Sbar; \astar)} \log \frac{C_3}{Z(\Sbar; \astar)}.
\end{split}
\end{equation}
Since $\KLahat + C_1 \leq -\log \frac{C_4}{\gamma_2} \leq 0$, $\KLastar \geq 0$, and also with~\eqref{eq:bound_1} and~\eqref{eq:bound_2}, we can relax~\eqref{eq:a23} into
\begin{equation}\small
\begin{split}
    & \quad \ \ \KLahat + C_1 - \frac{Z(\Sbar; \ahat)}{C_5 / \gamma_2} \log \frac{C_4}{Z(\Sbar; \ahat)} \\
    & \leq 
    \KLastar + \frac{C_2 e^{-\KLastar}}{C_2 e^{-\KLastar} + Z(\Sbar; \astar)} C_1 + \frac{Z(\Sbar; \astar)}{C_5 / \gamma_1} \log \frac{C_3}{Z(\Sbar; \astar)},
\end{split}
\end{equation}
which gives
\begin{equation}\small
    \KLahat
    \leq 
    \KLastar - \frac{Z(\Sbar; \astar)}{C_2 e^{-\KLastar} + Z(\Sbar; \astar)} C_1 + \frac{Z(\Sbar; \ahat)}{C_5 / \gamma_2} \log \frac{C_4}{Z(\Sbar; \ahat)}  + \frac{Z(\Sbar; \astar)}{C_5 / \gamma_1} \log \frac{C_3}{Z(\Sbar; \astar)}.
\end{equation}
Since $Z(\Sbar; \ahat) \leq \gamma_2 \leq \frac{C_4}{e}$, we have $Z(\Sbar; \ahat) \log \frac{C_4}{Z(\Sbar; \ahat)} \leq \gamma_2 \log \frac{C_4}{\gamma_2}$. 
Similarly, we have $Z(\Sbar; \astar) \log \frac{C_3}{Z(\Sbar; \astar)} \leq \gamma_2 \log \frac{C_3}{\gamma_2}$. Then we have
\begin{equation}
    \KLahat
    \leq 
    \KLastar - \frac{Z(\Sbar; \astar)}{C_2 e^{-\KLastar} + Z(\Sbar; \astar)} C_1 + \frac{\gamma_2^2}{C_5} \log \frac{C_4}{\gamma_2}  + \frac{\gamma_2^2}{C_5} \log \frac{C_3}{\gamma_2}.
\end{equation}
Since $C_1 \leq 0$, we can get
\begin{equation}\small
    \KLahat
    \leq 
    \KLastar + \frac{\gamma_2^2}{C_5} (-C_1) + \frac{\gamma_2^2}{C_5} \log \frac{C_4}{\gamma_2}  + \frac{\gamma_2^2}{C_5} \log \frac{C_3}{\gamma_2},
\end{equation}
which gives
\begin{equation}\small
    \KLahat
    \leq 
    \KLastar + \frac{2\gamma_2^2}{C_5} \log \frac{C_6}{\gamma_2},
\end{equation}
where $C_6 = \sqrt{C_3 C_4 e^{-C_1}}$.

\end{proof}

\begin{remark}
From Theorem~\ref{thrm:full} we see that $\KLahat$ is close to $\KLastar$ as long as $\gamma_2$ is small. One may notice that $\gamma_2$ cannot be arbitrarily small because from~\eqref{eq:bound_2} we have
\begin{equation}\small                 \frac{C_5}{\gamma_2} \leq C_2 e^{-\KLahat} \leq C_2,
\end{equation}
which means 
\begin{equation}\small
    \gamma_2 \geq \frac{C_5}{C_2} = C_4 e^{C_1}. 
\end{equation}
However, we can safely assume that
\begin{equation}\small
    C_1 = \log \frac{p^m(\theta^\ast)}{p^\ast (\theta^\ast)} \ll 0
\end{equation}
since $p^\ast$ is much more informative than $p^m$, especially when labeled data for the main task is scarce. This means $\gamma_2$ can be extremely small as long as $C_1$ is small, which makes $\KLahat$ close to $\KLastar$. Similarly, Assumption~\ref{assumption:2} can easily hold as long as $C_1$ is small.

\end{remark}

%%%%%%%%%%%%%%%%%%%%%%%%%%%%%%%%%%%%%%%%%%%%%%%%%%%%%%%%%%%%%%%%%%%%

\subsection{Sampling through Langevin Dynamics (P2)}

In \textbf{Samples (P2)} we use Langevin dynamics~\cite{neal2011mcmc,welling2011bayesian} to sample from the distribution $p^J$. Concretely, at each iteration, we update $\theta$ by
\begin{equation}
    \theta_{t+1} = \theta_t - \epsilon_t \nabla \mathcal{L}(\theta_t) + \eta_t,
\end{equation}
where $\mathcal{L}(\theta) \propto -\log p^J(\theta)$ is the joint loss, and $\eta_t \sim N(0, 2\epsilon_t)$ is a Gaussian noise. In this way, $\theta_t$ converges to samples from $p^J$, which can be used to estimate our optimization objective. However, since we normally use a mini-batch estimator $\hat{\mathcal{L}}(\theta)$ to approximate $\mathcal{L}(\theta)$, this may introduce additional noise other than $\eta_t$, which may make the sampling procedure inaccurate. In~\cite{welling2011bayesian} it is proposed to anneal the learning rate to zero so that the gradient stochasticity is dominated by the injected noise, thus alleviating the impact of mini-batch estimator. However we find in practice that the gradient noise is negligible compared to the injected noise (Table~\ref{tab:noise}). Therefore, we ignore the gradient noise and directly inject the noise $\eta_t$ into the updating step.

\begin{table}[h]
  \caption{Standard deviation of different types of noise. We find that the gradient noise is negligible compared to the injected noise.}
  \label{tab:noise}
  \vspace{1em}
  \centering
  \begin{tabular}{lc}
    \toprule
      & Standard deviation\\
    \midrule
    Gradient Noise & $\sim 10^{-6}$ \\
    Injected Noise & $\sim 10^{-3}$ \\
    \bottomrule
  \end{tabular}
\end{table}

%%%%%%%%%%%%%%%%%%%%%%%%%%%%%%%%%%%%%%%%%%%%%%%%%%%%%%%%%%%%%%%%%%%%

\subsection{Score Function and Fisher Divergence (P3)}

In \textbf{Partition Function (P3)} we propose to minimize
\begin{equation}\small
\label{eq:final_objective}
    \min_{\a} E_{\theta \sim p^J} \norm{\nabla \log p(\mathcal{T}_m | \theta) - \nabla \log p_{\a} (\theta)}_2^2
\end{equation}
as our final objective. Notice that
\begin{equation}\small
\begin{split}
    & \min_{\a} E_{\theta \sim p^J} \norm{\nabla \log p(\mathcal{T}_m | \theta) - \nabla \log p_{\a} (\theta)}_2^2 \\
    \Leftrightarrow \ \ & \min_{\a} E_{\theta \sim p^J} \norm{\nabla \log p^m(\theta) - \nabla \log p_{\a} (\theta)}_2^2 \\
    \Leftrightarrow \ \ & \min_{\a} E_{\theta \sim p^J} \norm{\nabla \log (p^m(\theta) \cdot p_{\a}(\theta)) - 2 \cdot \nabla \log p_{\a} (\theta)}_2^2 \\
    \Leftrightarrow \ \ & \min_{\a} E_{\theta \sim p^J} \norm{\nabla \log p^J(\theta) - \nabla \log p^2_{\a} (\theta)}_2^2 \\
    \Leftrightarrow \ \ & \min_{\a} F\infdivx{p^J(\theta)}{\frac{1}{Z^\prime (\a)} p^2_{\a}(\theta)},
\end{split}
\end{equation}
where $F\infdivx{p(\theta)}{q(\theta)} = E_{\theta \sim p} \norm{\nabla \log p(\theta) - \nabla \log q(\theta)}^2_2$ is the \emph{Fisher divergence}, and $Z^\prime (\a) = \int p^2_{\a} (\theta) d\theta$ is the normalization term. This means, by optimizing~\eqref{eq:final_objective}, we are actually minimizing the Fisher divergence between $p^J (\theta)$ and $\frac{1}{Z^\prime (\a)} p^2_{\a}(\theta)$. As pointed by~\cite{hu2018stein,liu2016kernelized}, Fisher divergence is stronger than KL divergence, which means  by minimizing $F\infdivx{p^J(\theta)}{\frac{1}{Z^\prime (\a)} p^2_{\a}(\theta)}$, the KL divergence $D_{KL}\infdivx{p^J(\theta)}{\frac{1}{Z^\prime (\a)} p^2_{\a}(\theta)}$ is also bounded near the optimum up to a small error. 

Therefore, optimizing~\eqref{eq:final_objective} is equivalent to minimizing $D_{KL}\infdivx{p^J(\theta)}{\frac{1}{Z^\prime (\a)} p^2_{\a}(\theta)}$. Notice that
\begin{equation}\small
\begin{split}
    & \min_{\a} D_{KL}\infdivx{p^J(\theta)}{\frac{1}{Z^\prime (\a)} p^2_{\a}(\theta)} \\
    \Leftrightarrow \ \ & \min_{\a} \int p^J (\theta) \log \frac{p^J (\theta)}{\frac{1}{Z^\prime (\a)} p^2_{\a}(\theta)} d\theta\\
    \Leftrightarrow \ \ & \min_{\a} \int p^J (\theta) \log \frac{\frac{1}{Z(\a)} p^m(\theta) p_{\a} (\theta)}{\frac{1}{Z^\prime (\a)} p^2_{\a}(\theta)} d\theta\\
    \Leftrightarrow \ \ & \min_{\a} \int p^J (\theta) \log \frac{ p^m(\theta)}{p_{\a}(\theta)} d\theta + \log \frac{Z^\prime (\a)}{Z(\a)}\\
    \Leftrightarrow \ \ & \min_{\a} \int p^J (\theta) \log \frac{ p^m(\theta)}{p_{\a}(\theta)} d\theta + \log \frac{\int p^2_{\a} (\theta) d\theta}{\int p^m(\theta) p_{\a}(\theta) d\theta}\\
\end{split}
\end{equation}
is different from~\eqref{eq:objective_approx} only on the $\log \frac{\int p^2_{\a} (\theta) d\theta}{\int p^m(\theta) p_{\a}(\theta) d\theta}$ term. To analyze the impact of this additional term, we assume that the likelihood function of each auxiliary task is a Gaussian, \ie, $p(\mathcal{T}_{a_k} | \theta) \propto N(\theta | \theta_k, \bm{\Sigma})$, with mean $\theta_k$ and covariance $\bm{\Sigma}$. Then we have $p_{\a} (\theta) = N(\theta | \sum_k \alpha_k \theta_k / K, \bm{\Sigma} / K)$ (note that $\sum_k \alpha_k = K$). In this case $\int p^2_{\a} (\theta) d\theta$ only depends on $\bm{\Sigma}$ and is invariant to $\a$. Thus optimizing~\eqref{eq:final_objective} is equivalent to
\begin{equation}\small
\label{eq:a35}
\begin{split}
    & \min_{\a} D_{KL}\infdivx{p^J(\theta)}{\frac{1}{Z^\prime (\a)} p^2_{\a}(\theta)} \\
    \Leftrightarrow \ \ & \min_{\a} \int p^J (\theta) \log \frac{ p^m(\theta)}{p_{\a}(\theta)} d\theta + \log \frac{\int p^2_{\a} (\theta) d\theta}{\int p^m(\theta) p_{\a}(\theta) d\theta} \\
    \Leftrightarrow \ \ & \min_{\a} \int p^J (\theta) \log \frac{ p^m(\theta)}{p_{\a}(\theta)} d\theta - \log \int p^m(\theta) p_{\a}(\theta) d\theta.
\end{split}
\end{equation}
Denote the optimal solution for~\eqref{eq:a35} by $\adagger$. Then we can build the connection between $\adagger$ and $\ahat$ by
\begin{equation}\small
    \int p^J (\theta) \log \frac{ p^m(\theta)}{p_{\adagger}(\theta)} d\theta - \log \int p^m(\theta) p_{\adagger}(\theta) d\theta 
    \leq 
    \int p^J (\theta) \log \frac{ p^m(\theta)}{p_{\ahat}(\theta)} d\theta - \log \int p^m(\theta) p_{\ahat}(\theta) d\theta.
\end{equation}
Since $\ahat$ minimizes $\int p^J (\theta) \log \frac{ p^m(\theta)}{p_{\a}(\theta)} d\theta$, which means $\int p^J (\theta) \log \frac{ p^m(\theta)}{p_{\ahat}(\theta)} d\theta \leq 
\int p^J (\theta) \log \frac{ p^m(\theta)}{p_{\adagger}(\theta)} d\theta$, we can get
\begin{equation}\small
    - \log \int p^m(\theta) p_{\adagger}(\theta) d\theta 
    \leq 
    - \log \int p^m(\theta) p_{\ahat}(\theta) d\theta,
\end{equation}
or 
\begin{equation}\small
    \int p^m(\theta) p_{\adagger}(\theta) d\theta 
    \geq 
    \int p^m(\theta) p_{\ahat}(\theta) d\theta,
\end{equation}
which gives
\begin{equation}\small
    \Sint p^m(\theta) p_{\adagger}(\theta) d\theta
    + \Sbarint p^m(\theta) p_{\adagger}(\theta) d\theta
    \geq 
    \Sint p^m(\theta) p_{\ahat}(\theta) d\theta 
    + \Sbarint p^m(\theta) p_{\ahat}(\theta) d\theta.
\end{equation}
Then we have
\begin{equation}\small
\begin{split}
    \Sint p^m(\theta) p_{\adagger}(\theta) d\theta
    &\geq 
    \Sint p^m(\theta) p_{\ahat}(\theta) d\theta 
    + \Sbarint p^m(\theta) p_{\ahat}(\theta) d\theta
    - \Sbarint p^m(\theta) p_{\adagger}(\theta) d\theta \\
    & \geq \Sint p^m(\theta) p_{\ahat}(\theta) d\theta - (\gamma_2 - \gamma_1).
\end{split}
\end{equation}
From Assumption~\ref{assumption:1} we have
\begin{equation}\small
    \frac{p^m(\theta^\ast) p_{\adagger}(\theta^\ast)}{p^\ast(\theta^\ast)}
    \geq \frac{p^m(\theta^\ast) p_{\ahat}(\theta^\ast)}{p^\ast(\theta^\ast)} - (\gamma_2 - \gamma_1),
\end{equation}
which gives
\begin{equation}\small
    \KLadagger = -\log \frac{p_{\adagger}(\theta^\ast)}{p^\ast(\theta^\ast)}
    \leq -\log (\frac{p_{\ahat}(\theta^\ast)}{p^\ast(\theta^\ast)} - \frac{\gamma_2 - \gamma_1}{p^m(\theta^\ast)})
    \leq -\log \frac{p_{\ahat}(\theta^\ast)}{p^\ast(\theta^\ast)} + \frac{\gamma_2 - \gamma_1}{p^m(\theta^\ast)},
\end{equation}
or 
\begin{equation}\small
    \KLadagger \leq \KLahat + \frac{\gamma_2}{C_2}.
\end{equation}
After combining with Theorem~\ref{thrm:full}, we have
\begin{equation}\small
    \KLadagger \leq \KLastar + \frac{2\gamma_2^2}{C_5} \log \frac{C_6}{\gamma_2} + \frac{\gamma_2}{C_2}.
\end{equation}
This means by optimizing our final objective~\eqref{eq:final_objective}, the KL divergence $\KLadagger$ is also bounded near the optimal value, which provides a theoretical justification of our algorithm.

%%%%%%%%%%%%%%%%%%%%%%%%%%%%%%%%%%%%%%%%%%%%%%%%%%%%%%%%%%%%%%%%%%%%

\subsection{Tips for Practitioners}
\label{appendix_tips}

In Section~\ref{sec:algorithm}, we propose a two-stage algorithm, where we update the task weights with Langevin dynamics in the first stage, and then udpate the model with fixed task weights in the second stage. However, we find in practice that we can also find the similar task weights if we turn off the Langevin dynamics and directly sample from regular SGD. Therefore, we can further simplify the algorithm by removing the Langevin dynamics and merge the two stage, \ie, update task weights and model parameters at the same time until convergence. This simplified version is summarized in Algorithm~\ref{alg_simplified}. 

\begin{algorithm}[h]
\caption{ARML (simplified version)}
\label{alg_simplified}
\begin{algorithmic}

\STATE \textbf{Input:} main task data $\mathcal{T}_m$, auxiliary task data $\mathcal{T}_{a_k}$, initial parameter $\theta_0$, initial task weights $\bm{\alpha}$
\STATE \textbf{Parameters:} learning rate of $t$-th iteration $\epsilon_t$, learning rate for task weights $\beta$
%\STATE \textbf{Outputs:} Final parameter $\theta_T$
\\~\\

\FOR{iteration $t = 1$ to $T$}
    \STATE $\theta_t \gets \theta_{t-1} - \epsilon_t (- \nabla \log p(\mathcal{T}_m | \theta_{t-1}) - \sum_{k=1}^K \alpha_k \nabla \log p(\mathcal{T}_{a_k} | \theta_{t-1})) + \eta_t$
    \STATE $\bm{\alpha} \gets \bm{\alpha} - \beta \nabla_{\bm{\alpha}} \norm{\nabla \log p(\mathcal{T}_m | \theta_t) - \sum_{k=1}^K \alpha_k \nabla \log p(\mathcal{T}_{a_k} | \theta_t)}^2_2$ 
    \STATE Project $\bm{\alpha}$ back into $\mathcal{A}$
\ENDFOR

\end{algorithmic}
\end{algorithm}

%%%%%%%%%%%%%%%%%%%%%%%%%%%%%%%%%%%%%%%%%%%%%%%%%%%%%%%%%%%%%%%%%%%%

\section{Experimental Settings}

For all results, we repeat experiments for three times and report the average performance. Error bars are reported with CI=95\%. In our algorithm, the only hyperparameter is the learning rate $\beta$ of task weights. Specifically, we find the results insensitive to the choice of $\beta$. Therefore, we randomly choose $\beta \in [0.0005, 0.05]$, for a trade-off between steady training and fast convergence. We use PyTorch~\cite{paszke2019pytorch} for implementation.

\subsection{Semi-supervised Learning}

For semi-supervised learning, we use two datasets, CIFAR10~\cite{krizhevsky2009learning} and SVHN~\cite{netzer2011reading}. For CIFAR10, we follow the standard train/validation split, with 45000 images for training and 5000 for validation. Only 4000 out of 45000 training images are labeled. For SVHN, we use the standard train/validation split with 65932 images for training and 7325 for validation. Only 1000 out of 65392 images are labeled. Both datasets can be downloaded from the official PyTorch torchvision library (\url{https://pytorch.org/docs/stable/torchvision/index.html}). Following~\cite{oliver2018realistic}, we use WRN-28-2 as our backbone, \ie, ResNet~\cite{he2016deep} with depth 28 and width 2, including batch normalization~\cite{ioffe2015batch} and leaky ReLU~\cite{maas2013rectifier}. We train our model for 200000 iterations, using Adam~\cite{kingma2014adam} optimizer with batch size of 256 and learning rate of 0.005 in first 160000 iterations and 0.001 for the rest iterations.

For implementation of self-supervised semi-supervised learning (S4L), we follow the settings in the original paper~\cite{zhai2019s4l}. Note that we make two differences from~\cite{zhai2019s4l}: (i) for steadier training, we use the model with time-averaged parameters~\cite{tarvainen2017mean} to extract feature of the original image, (ii) To avoid over-sampling of negative samples in triplet-loss~\cite{arora2019theoretical}, we only put a loss on the cosine similarity between original feature and augmented feature.

\subsection{Multi-label Classification}

For multi-label classification, we use CelebA~\cite{liu2015deep} as our dataset. It contains 200K face images, each labeled with 40 binary attributes. We cast this into a multi-label classification problem, where we randomly choose one attribute as the main classification task, and other 39 as auxiliary tasks. We randomly choose 1\% images as labeled images for main task. The dataset is available at \url{http://mmlab.ie.cuhk.edu.hk/projects/CelebA.html}. We use ResNet18~\cite{he2016deep} as our backbone. We train the model for 90 epochs using SGD solver with batch size of 256 and scheduled learning rate of 0.1 initially and 0.1$\times$ shrinked every 30 epochs.

\subsection{Domain Generalization}

Following the literature~\cite{carlucci2019domain,asadi2019towards}, we use PACS~\cite{li2017deeper} as our dataset for domain generalization. PACS consists of four domains (photo, art painting,
cartoon and sketch), each containing 7 categories (dog, elephant, giraffe, guitar, horse, house and person). The dataset is
created by intersecting classes in Caltech-256~\cite{griffin2007caltech}, Sketchy~\cite{sangkloy2016sketchy}, TU-Berlin~\cite{eitz2012humans} and Google Images. Dataset can be downloaded from \url{http://sketchx.eecs.qmul.ac.uk/}. Following
protocol in~\cite{li2017deeper}, we split the images from training domains to 9 (train) : 1 (val) and test on the whole target
domain. We use a simple data augmentation protocol by randomly cropping the images to 80-100\% of original sizes and
randomly apply horizontal flipping. We use ResNet18~\cite{he2016deep} as our backbone. Models are trained with SGD solver, 100 epochs, batch size 128.
Learning rate is set to 0.001 and shrinked down to 0.0001 after 80 epochs.

\end{document}